\documentclass[11pt]{article}

\usepackage{fullpage,url}
\usepackage{times}

\usepackage{amsthm,amsfonts,amsmath,amssymb,epsfig,color,float,graphicx,verbatim}
\usepackage{enumitem}
\usepackage{algorithm}
\usepackage{algpseudocode}

\usepackage{footmisc}
\usepackage{wrapfig}
\usepackage{subcaption}
\usepackage{bbm, bm}
\usepackage{natbib}

\usepackage{hyperref}
\hypersetup{
	colorlinks   = true, 
	urlcolor     = blue, 
	linkcolor    = blue, 
	citecolor   = blue, 
    breaklinks=true
}
\usepackage{breakcites}

\usepackage{nicefrac}
\usepackage{thm-restate}
\usepackage{authblk}
\usepackage{microtype}

\setcitestyle{authoryear,round,citesep={;},aysep={,},yysep={;}}

\newtheorem{lemma}{Lemma}
\newtheorem{corollary}{Corollary}
\newtheorem{definition}{Definition}
\newtheorem{remark}{Remark}
\newtheorem{assumption}{Assumption}

\newcommand{\bepsilon}{\boldsymbol{\epsilon}}

\newcommand{\Xcal}{\mathcal{X}}

\newcommand{\Hcal}{\mathcal{H}}

\newcommand{\Ncal}{\mathcal{N}}

\newcommand{\secref}[1]{Sec.~\ref{#1}}

\newcommand{\figref}[1]{Fig.~\ref{#1}}
\renewcommand{\eqref}[1]{Eq.~(\ref{#1})}
\newcommand{\defref}[1]{Def.~(\ref{#1})}
\newcommand{\lemref}[1]{Lemma~\ref{#1}}
\newcommand{\corref}[1]{Corollary~\ref{#1}}
\newcommand{\thmref}[1]{Thm.~\ref{#1}}

\newcommand{\appref}[1]{Appendix~\ref{#1}}
\newcommand{\assref}[1]{Assumption ~\ref{#1}}

\newcommand{\Real}{\mathbb{R}}
\newcommand{\R}{\mathbb{R}}
\newcommand{\N}{\mathbb{N}}

\newcommand{\E}{\mathbb{E}}
\newcommand{\Sphere}{\mathbb{S}}

\newcommand{\norm}[1]{\left\lVert#1\right\rVert}

\newcommand{\abs}[1]{\left\vert#1\right\rvert}

\newcommand{\tr}{\text{tr}}

\newcommand{\st}{\text{ s.t.\ }}
\newcommand{\diag}{\text{diag}}
\newcommand{\err}{\rho}
\newcommand{\bigo}{\mathcal{O}}
\newcommand{\polylog}{\text{polylog}}

\newcommand{\leqk}{{\leq k}}
\newcommand{\geqk}{{> k}}

\newcommand{\geql}{{> l}}

\newcommand{\geqn}{{> n}}
\DeclareMathOperator*{\argmax}{arg\,max}
\DeclareMathOperator*{\argmin}{arg\,min}

\newcommand{\vv}{\mathbf{v}}

\newcommand{\x}{\mathbf{x}}
\newcommand{\y}{\mathbf{y}}

\newcommand{\krmat}{\mathbf{K}}
\newcommand{\reg}{\gamma_n}

\title{Generalization in Kernel Regression Under Realistic Assumptions}
\author{
Daniel Barzilai
\qquad
Ohad Shamir
\vspace{3pt}
\\
Weizmann Institute of Science \\
\texttt{\{daniel.barzilai,ohad.shamir\}@weizmann.ac.il}  
}
\date{}

\begin{document}
\maketitle

\begin{abstract}
    It is by now well-established that modern over-parameterized models seem to elude the bias-variance tradeoff and generalize well despite overfitting noise. Many recent works attempt to analyze this phenomenon in the relatively tractable setting of kernel regression. However, as we argue in detail, most past works on this topic either make unrealistic assumptions, or focus on a narrow problem setup. This work aims to provide a unified theory to upper bound the excess risk of kernel regression for nearly all common and realistic settings. Specifically, we provide rigorous bounds that hold for common kernels and for any amount of regularization, noise, any input dimension, and any number of samples. Furthermore, we provide relative perturbation bounds for the eigenvalues of kernel matrices, which may be of independent interest. These reveal a self-regularization phenomenon, whereby a heavy tail in the eigendecomposition of the kernel provides it with an implicit form of regularization, enabling good generalization. When applied to common kernels, our results imply benign overfitting in high input dimensions, nearly tempered overfitting in fixed dimensions, and explicit convergence rates for regularized regression. As a by-product, we obtain time-dependent bounds for neural networks trained in the kernel regime. 
\end{abstract}

\section{Introduction}

It is by now well-established that various families of highly over-parameterized models tend to generalize well, even when perfectly fitting noisy data \citep{zhang2021understanding, belkin2019reconciling}. This phenomenon seemingly contradicts the classical intuition of the bias-variance tradeoff, and motivated a large literature attempting to explain it \citep{bartlett2020benign, hastie2022surprises}. 

In particular, a long series of works attempted to understand this phenomenon in the context of kernel methods \citep{liang2020just, mei2022generalization, xiao2022precise, mallinar2022benign}. This is due both to their classical importance and their relation to over-parameterized neural networks via the Neural Tangent Kernel (NTK) and Gaussian Process Kernel (GPK, also known as NNGP) \cite{lee2017deep, jacot2018neural}. However, there is still a large gap between empirical observations and current theoretical analysis. As we argue in detail in \secref{sec:preliminaries}, past works tend to either make unrealistic assumptions (often inspired by the analysis of \emph{linear} regression) that do not hold for common kernels of interest, or are limited to a very narrow problem setup. This is not just a technical limitation, but rather, as we will show, may result in an inaccurate analysis for common kernels in practice. In this paper, we provide simple, sharp, and rigorous upper bounds for the generalization error of kernel regression, which hold under realistic assumptions and can be applied to a wide range of kernels and settings.

Specifically, we demonstrate that many kernels have a built-in \emph{self-regularization} property, meaning that the structure of the kernel provides an implicit form of regularization. This property is characterized through novel relative deviation bounds on the eigenvalues of kernel matrices, which may be of independent interest and may be useful in many other settings.

We then apply these tools to analyze the generalization performance of regularized and un-regularized kernel regression. Self-regularization causes the kernel to learn a function that generalizes well, even if it can interpolate the data. As such, we provide upper bounds for the excess risk (and its bias and variance components) regardless of the amount of explicit regularization. Importantly, our mild assumptions allow us to apply these bounds to common kernels, including NTKs (and hence provide insights on generalization in neural networks). Specifically, our main results and insights include the following:
\begin{enumerate}
    \item \textbf{Relative concentration bounds for the eigenvalues of kernel matrices (\thmref{thm:ker_eigenvalues}).} We derive both upper and lower bounds for the eigenvalues of kernel matrices under very mild assumptions which hold for common kernels. In particular, this highlights a self-regularization phenomenon whereby the eigenvalues of the kernel matrix behave as if one added an explicit regularization term to the training objective.
    \item \textbf{A general-purpose upper bound for the excess risk in kernel regression (\thmref{thm:bound_gen}).} The assumptions of this bound are very mild, and the bound can thus be applied to common kernels in a variety of settings. The bound is sharp without further assumptions, and characterizes both the bias and variance up to universal constants. In particular, no assumption is made on the regularization strength, amount of noise, input dimension, or number of samples.
    \item \textbf{Benign overfitting in high input dimensions (\thmref{thm:highdim}),} meaning that the excess risk goes to zero despite the presence of noise and lack of explicit regularization. In such a high dimensional setting, the frequencies that can be learned are limited, thus preventing any harmful overfitting. In particular, our results apply to the NTK, showing benign overfitting (and the corresponding convergence rates) for neural networks in the kernel regime when the input dimension is large.
    \item \textbf{Nearly Tempered overfitting in fixed input dimensions (\thmref{thm:min_norm_poly}),} meaning that the bias goes to zero, and the variance cannot diverge too quickly. As such, when the amount of noise is relatively small, this implies a good excess risk despite a possibly harmful overfitting of noise. As far as we know, this is the first rigorous upper bound for unregularized kernel regression (i.e., min-norm interpolator) in the fixed dimensional setting for generic kernels.
    \item \textbf{Learning rates for regularized kernel regression  (\thmref{thm:fixed_dimensional}),} where we bound the bias and variance as a function of the regularization strength. In particular, through a connection with gradient flow, this gives convergence rates for neural networks trained in the kernel regime. 
\end{enumerate}

Overall, we hope that our paper will contribute to the development of a rigorous general theory analyzing overfitting in kernel regression and, more generally, in over-parameterized models under minimal and realistic assumptions. 

The paper is structured as follows: In \secref{sec:preliminaries}, we formally present our settings and explain the issues with past works. In \secref{sec:eigenvalues} we present our eigenvalue bounds, and in \secref{sec:krr} our bias and variance bound for kernel regression. \secref{sec:applications} specializes to specific cases showcasing the utility of our previous results. In \secref{sec:neural_nets} we discuss the implications of our results for neural networks. All of our proofs are given in the appendix. Namely, the proofs for \secref{sec:eigenvalues}, \secref{sec:krr} and \secref{sec:applications} are given in \appref{app:eigenvalues}, \appref{app:krr} and \appref{app:applications} respectively. Further appendices are referred to from the text as needed.

\section{Preliminaries and Discussion of Past Works}\label{sec:preliminaries}
\subsection{Problem Set-Up}
Let $\Xcal$ be some input space, $\mu$ some measure over $\Xcal$ and $K:\Xcal\times\Xcal\to \R$ be a positive definite kernel over $\Xcal$. We assume that $K$ is a Mercer kernel, meaning that it admits a spectral decomposition 
\begin{align}\label{eq:mercer}
    K(\x,\x') = \sum_{i=1}^\infty \lambda_i \psi_i(\x)\psi_i(\x'),
\end{align}
where $\lambda_i \geq 0$ are the non-negative eigenvalues (not necessarily ordered), and the eigenfunctions $\psi_i$ form an orthonormal basis in $L^2_{\mu}(\Xcal)$ (the space of square-integrable functions w.r.t. $\mu$). This is a very mild assumption, as it holds for (but is not limited to) the cases where $\mu$ is a probability measure, and either $\Xcal$ is compact or $K$ is bounded and continuous \citep{steinwart2012mercer}. Let $p\in\N\cup\{\infty\}$ denote the number of non-zero eigenvalues, and w.l.o.g let $\phi(\x):=\left(\sqrt{\lambda_i}\psi_i(\x) \right)_{i=1}^p$ be the non-zero features (with $\lambda_i>0$) and $\psi(\x):=\left(\psi_i(\x) \right)_{i=1}^p$. Since $\E_x[\psi(\x)\psi(\x)^\top]=I$, it is straightforward to verify that the features admit a diagonal and invertible (uncentered) covariance operator given by 
\begin{align}\label{eq:Sigma}
\Sigma:=\mathbb{E}_{\x}\left[\phi(\x)\phi(\x)^\top \right] = \begin{bmatrix}
\lambda_1 &  & 0\\
 & \lambda_2 & \\
 0 & & \ddots
\end{bmatrix}.
\end{align}
The features are related to the eigenfunctions by $\phi(\x)=\Sigma^{\nicefrac{1}{2}}\psi(\x)$, and to the kernel by $K(\x,\x')=\langle \phi(\x), \phi(\x') \rangle$ where the dot product is the standard one. 

We will always work in the over-parameterized setting, meaning that throughout the paper, we assume that $p\geq n$. Since oftentimes $p=\infty$, our bounds will not explicitly depend on $p$ (only implicitly through the eigenvalues of $\Sigma$). 

Let $X=\{\x_1,...,\x_n\} \subseteq \Xcal$ be a set of $n$ training points drawn i.i.d from $\mu$. Let $f^*\in L^2_{\mu}(\Xcal)$ be some target function, and $y_i = f^*(\x_i) + \epsilon_i$ be the response variable, where $\epsilon_i$ is any i.i.d noise with mean $0$ and variance $\sigma_\epsilon^2$.  Given some regularization parameter $\reg > 0$, \emph{Kernel Ridge Regression} (KRR) corresponds to minimizing the objective
\begin{align}\label{def:krr}
    \min_{f\in\Hcal} \frac{1}{n}\sum_{i=1}^n \left(f(\x_i) - y_i\right)^2 + \reg\norm{f}_{\Hcal}^2, 
\end{align}
where $\Hcal$ is the RKHS of $K$, consisting of functions of the form $f(\x)=\langle \theta, \phi(\x)\rangle$ with $\norm{\theta}_2<\infty$. Letting $\y=(y_1,\ldots,y_n)^\top$, the minimizer of the KRR problem in  \eqref{def:krr} is given by $\hat{f}(\x) = \langle \hat{\theta}(\y), \phi(\x) \rangle$ with:

\noindent
\begin{align}\label{eq:krr_solution}
    \hat{\theta}(\y) = \phi(X)^\top (\krmat+n\reg I)^{-1}\y,
\end{align}
where $\krmat=K(X,X)=\left(K(\x_i,\x_j)\right)_{i,j=1}^n$ is the kernel matrix, and using infinite matrix notation, $\phi(X):=[\phi(\x_1),\ldots,\phi(\x_n)]^\top \in \R^{n\times p}$ are the training features. As $\reg\to 0$, $\hat{\theta}$ tends to the \emph{min-norm interpolator}:
\begin{align}\label{eq:min_norm_solution}
    \hat{\theta}(y) = \arg\min_\theta \norm{\theta}_{\Hcal}\st \y=\phi(X)\theta.
\end{align}

We can decompose the target function as $f^*(\x) = \langle \theta^*, \phi(\x) \rangle + P^{\perp} f^*$ where $\theta^*\in\R^p$ and $P^{\perp}$ is the orthogonal projection onto the space spanned by the eigenfunctions with $0$ eigenvalues (from \eqref{eq:mercer}). In particular, if the kernel function $K$ from \eqref{eq:mercer} has no zero eigenvalues, then $P^{\perp}f^*=0$. By the orthonormality of $\psi_i$, it holds that $\norm{f^*}_{ L^2_{\mu}(\Xcal)}=\norm{\Sigma^{\nicefrac{1}{2}}\theta^*}_2 + \norm{P^{\perp}f^*}_{ L^2_{\mu}(\Xcal)}$. We do not require $f^*$ to be in the RKHS.

We will define the excess risk of KRR as:
\begin{align}\label{eq:risk}
    R\left(\hat\theta(y)\right):=\mathbb{E}_{\x,\epsilon}\left[\left( \langle \hat{\theta}(\y), \phi(\x) \rangle - f^*(\x)\right)^2\right] = \mathbb{E}_{\x,\epsilon}\left[\left\langle \hat\theta(\y) - \theta^*, \phi(\x) \right\rangle^2\right] + \norm{P^{\perp}f^*}_{ L^2_{\mu}(\Xcal)}^2.
\end{align}
Some authors equivalently analyze the risk, namely the expected error w.r.t. the noisy labels, which is equal to  $\sigma_{\epsilon}^2 + R\left(\hat\theta(\y)\right)$. By linearity, the predictor can be decomposed as $\hat{\theta}(\y)=\hat{\theta}(\phi(X)\theta^*) + \hat{\theta}(\bepsilon)$ where $\bepsilon\in \R^n$ is the noise on the training set. Using this, the fact that the noise is independent of $\x$, and the definition of $\Sigma$ from \eqref{eq:Sigma}, the excess risk from \eqref{eq:risk} can be decomposed in terms of bias, variance, and an approximation error as:
\begin{align}\label{eq:bias_variance}
R\left(\hat\theta(y)\right) = \underset{:=B}{\underbrace{\norm{\hat{\theta}(\phi(X)\theta^*)-\theta^*}^2_{\Sigma}}} + \underset{:=V}{\underbrace{\mathbb{E}_{\epsilon}\left[\norm{\hat{\theta}(\bepsilon)}^2_{\Sigma}\right]}} + \norm{P^{\perp}f^*}_{ L^2_{\mu}(\Xcal)}^2,
\end{align}
where $\norm{\x}_\Sigma=\sqrt{\x^\top \Sigma\x}$. 

\subsection{Issues With Past Works}
There is a vast literature on KRR and linear regression, with many interesting results under various assumptions and settings. However, perhaps surprisingly, there does not appear to be a unified theory that can provide upper bounds for the excess risk of kernel regression for common kernels and for \emph{any} amount of regularization, noise, any input dimension, and any number of samples. We now detail a few aspects of how current bounds are insufficient.

\begin{itemize}
    \item \textbf{Assumptions That Do Not Hold: } Many works rely on assumptions that are common or reasonable for analyzing \emph{linear} regression. However, as we argue below, they are generally inapplicable for kernel regression. These assumptions include that the features $\phi_i(\x)$ are Gaussians \citep{spigler2020asymptotic, jacot2020kernel, cui2021generalization}, the eigenfunctions $\psi_i(\x)$ (sometimes called covariates) are sub-Gaussian, i.i.d, finite dimensional, and/or have mean $0$ \citep{bartlett2020benign, hastie2022surprises, cheng2022dimension, tsigler2023benign, bach2023high, cheng2023theoretical} or various nonrigorous assumptions common in the statistical physics literature \citep{bordelon2020spectrum, gerace2020generalisation, canatar2021spectral, simon2021eigenlearning, mallinar2022benign}. Unfortunately, none of these assumptions hold for common kernels, making such works incapable of providing rigorous results in common settings. As a simple example, suppose our inputs are one-dimensional standard Gaussians $x\sim \Ncal(0,1)$ and let $K(x,y)=\exp\left(-\gamma(x-y)^2\right)$ be a Gaussian (RBF) kernel. Such kernels have known Mercer decompositions \citep{fasshauer2011positive} with eigenfunctions $\psi_i$ given by Hermite polynomials. We show in \appref{appendix:rbf} that if we pick for simplicity $\gamma=\frac{3}{8}$, then for any $p \geq 3$, the moments of $\psi_i$ diverge as
    \begin{align}\label{eq:diverging_moments}
        \forall p \geq 3, ~~~ \left(\E\left[\abs{\psi_i(x)}^p\right]\right)^{\nicefrac{1}{p}} \geq \Omega_i\left(\exp\left(\frac{p-2}{4} \cdot i\right)\right) \underset{i\to\infty}{\longrightarrow}\infty.
    \end{align}
    
    Thus, for the classical RBF kernel with Gaussian inputs, not only is $\psi(\x)$ not sub-Gaussian, but all moments $\geq 3$ diverge. Another simple example is given with inputs distributed uniformly on the unit sphere $\Sphere^{d-1}$, and dot product kernels such as RBF, Laplace and NTK. Under this setting, $\psi_i(\x)$ are given by spherical harmonics, for which even in the case of $d=3$ the third moments diverge as $i\to \infty$ \citep{han2016spherical}. Additionally, for dot product kernels, $\psi_i$ are definitely not i.i.d across $i$, $\psi_1$ is generally constant and not mean $0$, and $p$ may be $\infty$ (see \appref{appendix:dot-product} for more details.)
    Furthermore, these assumptions are not only unrealistic but also lead to inaccurate predictions. Specifically, they induce concentration inequalities (e.g bounding the eigenvalues of the empirical covariance matrix) which are tighter than one can typically expect, resulting in risk bounds that may be over-optimistic (see \figref{fig:low_dimensional}). By contrast, we work under very mild and realistic assumptions, and we do not know of any interesting kernel for which our analysis is not applicable. \\ 

    \item \textbf{Limitation to a Specific Setting: }
    The literature seems to be split into several categories, with different works focusing on incompatible settings. These include: 
    \begin{itemize}
        \item \emph{"High-Dimensional" vs. "Fixed-Dimensional"}: Many works assume that the input dimension $d$ and the number of samples $n$ both tend towards infinity at a fixed ratio $n=d^{\tau}$ for some $\tau > 0$ \citep{dobriban2018high, liang2020just, wu2020optimal, richards2021asymptotics, ghorbani2021linearized, li2021towards, hu2022universality, montanari2022interpolation, mei2022generalization, misiakiewicz2022spectrum, xiao2022precise}. By contrast, other lines of work assume a fixed $d$ and $n\to\infty$ \citep{caponnetto2007optimal, steinwart2009optimal, li2023statistical, cui2021generalization, li2023asymptotic}. The techniques and assumptions used by these two lines of work are inherently different, and make the results from the high-dimensional works inapplicable for fixed $d$ and vice versa. For example, high-dimensional works typically rely on tools from random matrix theory, which require $d$ and $n$ to be tied and are inapplicable for a fixed $d$. By contrast, low-dimensional works have bounds that depend on the properties of the fixed RKHS, and often assume a fixed polynomial decay for the eigenvalues $\lambda_i$. This not only excludes kernels with an exponential decay such as RBF \citep{minh2006mercer} but is also problematic, for example, for analyzing the NTK with high-dimensional inputs, since the polynomial decay only begins when the eigenvalue index is $i\gg \text{poly}(d)$ \citep{cao2019towards}.  By contrast, we obtain bounds that are relevant for \emph{any} $d,n$, regardless of the ratio between them, and in particular, capture interesting phenomena in these two regimes.

        \item \emph{Regularized vs. Unregularized: } Several works are limited to either the regularized case \citep{caponnetto2007optimal, steinwart2009optimal, fischer2020sobolev, lin2020optimal} or the unregularized case (a.k.a min-norm interpolation) \citep{bartlett2020benign, liang2020just, hastie2022surprises}. This distinction is of course unwanted, and our results provide bounds that can handle both and make the role of the regularization explicit. 
        
        \item \emph{Noisy vs. Noiseless:} \citet{cui2021generalization} noted a discrepancy between rates obtained in a noisy setting (when  $\sigma_{\epsilon}>0$) \citep{caponnetto2005fast, steinwart2009optimal} vs. a noiseless setting (when  $\sigma_{\epsilon}=0$) \citep{spigler2020asymptotic}. Furthermore, quantifying the effect of the noise is important since even when $\sigma_{\epsilon}>0$, one may still obtain a small excess risk if the noise is small. Recent works in the fixed dimensional setting still only manage to provide upper bounds in the noiseless case \citep{li2023asymptotic}. Our analysis handles both cases, separating the bias and variance, and upper bounding both of these separately.

    \end{itemize}
\end{itemize}

There are also prior works that bound the eigenvalues of kernel matrices similarly to what we do here. \citet{braun2005spectral, rosasco2010learning, valdivia2018relative} provide generic bounds; however, they are not sufficiently strong for many applications and, in particular, often do not yield nontrivial bounds for the smallest eigenvalue of the kernel matrix. As we shall see, this will be crucial for our analysis. \citet{fan2020spectra, montanari2022interpolation} provide lower bounds for the smallest eigenvalue when the input dimension is linear in the number of samples and tends towards infinity. For fully-connected NTKs, \citet{oymak2020toward, wang2021deformed} provide bounds for two-layer networks, and \citet{nguyen2021tight} provide bounds for deep networks for large input dimensions. \citet{belkin2018approximation} gives bounds for radial kernels such as RBF.

\subsection{Additional Notations and Definitions}
We use the subscripts $\leqk$ and $\geqk$ to denote the first $1,\ldots,k$ and $k+1,k+2,\ldots$ coordinates of a vector respectively. So for example, $\phi_\leqk(X)$ is an $n\times k$ matrix. Analogously, we let $\krmat_{\leqk} := \phi_\leqk(X)\phi_\leqk(X)^T$ and $\krmat_{\geqk} := \phi_\geqk(X)\phi_\geqk(X)^T$. For an operator $T$, we use $\mu_i(T)$ to denote its $i$'th largest eigenvalue (where we allow repeated eigenvalues, i.e $\mu_i(T)=\mu_j(T)$ is allowed). We use this notation to avoid confusion with the eigenvalues $\lambda_i$ of $\Sigma$. Unless stated otherwise, $\norm{\cdot}$ is the standard $\ell^2$ norm for vectors, and operator norm for operators. We use the standard big-O notation, with $\bigo(\cdot)$, $\Theta(\cdot)$ and $\Omega(\cdot)$ hiding absolute constants that do not depend on problem parameters, $\tilde{\bigo}(\cdot)$ and $\tilde{\Omega}(\cdot)$ hiding absolute constants and additional logarithmic factors. We may make the problem parameters explicit, e.g $\bigo_{n,d}$ to mean up to constants that do not depend on $n$ or $d$.

As in \citet{bartlett2020benign}, for any $k\in\N$, we define two highly related notions of the effective rank of $\Sigma_{\geqk}$ as:
\begin{definition}\label{def:effective_rank}
    \[
    r_k:=r_k(\Sigma):=\frac{\tr(\Sigma_{\geqk})}{\norm{\Sigma_{\geqk}}}, ~~~~~~~~ R_k:=R_k(\Sigma)=\frac{\tr(\Sigma_{\geqk})^2}{\tr\left(\Sigma_{\geqk}^2\right)}.
    \]
\end{definition}
$r_k$ is the common definition of effective rank, and $R_k$ is related to $r_k$ via $r_k \leq R_k \leq r_k^2$ \citep{bartlett2020benign}[Lemma 5].

Typically, one must assume something on $\psi(\x)$ to obtain various concentration inequalities, meaning that the kernel matrix and empirical covariance matrix will behave as they are "supposed to". Perhaps the most common assumption in previous works is that $\psi(\x)$ is sub-Gaussian, requiring the moments of $\psi_i(\x)$ to be sufficiently well-behaved for every $i$. Unfortunately, as discussed earlier, this does not hold for many common kernels, even when the input distribution is "nice." In order to overcome this issue, we present a framework for analyzing kernels under only a mild heavy-tailed condition which can be shown to hold for many common kernels. In particular, we wish that quantities concerning the features will be related to their expected values by a multiplicative constant. By the orthonormality of $\psi_i$, for any $k\in\N$ one has that $\E[\norm{\psi_{\leqk}(\x)}^2]= k, \E[\norm{\phi_{\geqk}(\x)}^2] = \tr\left(\Sigma_{\geqk}\right)$ and $\E\left[\norm{\Sigma_{\geqk}^{\nicefrac{1}{2}}\phi_{\geqk}(\x)}^2\right]=\tr\left(\Sigma_{\geqk}^{2}\right)$. We quantify the distance of the quantities from their expected values by the following definitions:

\begin{definition}\label{def:eigen_combined}
Given $k\in\N$, let $\beta_k \geq \alpha_k \geq 0$ be defined as follows:
\begin{equation}\label{def:eigen_lower}
    \alpha_k:=\inf_{\x}\left\{\frac{\norm{\phi_{\geqk}(\x)}^2}{\tr\left(\Sigma_{\geqk}\right)}\right\}.
\end{equation}
\begin{equation}\label{def:eigen_upper}
    \beta_k:=\sup_{\x}\max\left\{\frac{\norm{\psi_{\leqk}(\x)}^2}{k}, \frac{\norm{\phi_{\geqk}(\x)}^2}{\tr\left(\Sigma_{\geqk}\right)}, \frac{\norm{\Sigma_{\geqk}^{\nicefrac{1}{2}}\phi_{\geqk}(\x)}^2}{\tr\left(\Sigma_{\geqk}^{2}\right)} \right\}.
\end{equation}
Where the $\sup$ and $\inf$ are for a.s any $\x$.
\end{definition}

For each term in these definitions, the denominator is the expected value of the numerator, so $\alpha_k$ and $\beta_k$ quantify how much the features behave as they are "supposed to". Since $\inf\leq \E \leq \sup$, one always has $0\leq \alpha_k \leq 1 \leq \beta_k$. Upper bounding $\beta_k$ is often easy, and common examples for kernels with $\beta_k=\bigo_k(1)$ include dot-product kernels such as NTK and polynomial kernels, shift-invariant kernels, random features and kernels with bounded eigenfunctions $\norm{\psi(\x)}_{\infty}<\infty$. $\alpha_k$ can also be lower bounded as $\Omega_k(1)$ for many kernels (e.g dot-product kernels); however, a lower bound on $\alpha_k$ may sometimes be more difficult, and as such, many of our bounds will not require any control of $\alpha_k$. Nevertheless, when $\alpha_k>0$, in some cases, stronger bounds will be available. We defer a more complete discussion of these definitions, their relation to common kernels, and our claims in this paragraph to \appref{appendix:example_kernels}. Overall, for sufficiently "nice" kernels, one should think of $\alpha_k$ and $\beta_k$ as generally being $\Theta_k(1)$. For the bounds in this paper, we will not need to control $\alpha_k$ and $\beta_k$ for every value of $k$, but rather $k$ can be arbitrarily chosen. 
\begin{remark}\label{remark:high_prob}
    \defref{def:eigen_lower} and \defref{def:eigen_upper} are stated for a.s any $\x$. However, one can weaken the definition for $\alpha_k$ to the training set, so that w.p at least $1-\delta_{k}$, $\min_{\x \in \{\x_1,\ldots, \x_n\}}\left\{\frac{\norm{\phi_{\geqk}(\x)}^2}{\tr\left(\Sigma_{\geqk}\right)}, \right\} \geq \alpha_k$. In such a case, all bounds that depend on $\alpha_k$ would still hold with probability $1-\delta_k$.
\end{remark}
In some cases, we will need to make the control of $\beta_k$ explicit via the following regularity assumption.

\begin{assumption}\label{assumption:good_beta}
    Either the feature dimension $p$ is finite, or there exists some sequence of natural numbers $(k_i)_{i=1}^\infty\subseteq \N$ with $k_i\underset{i\to\infty}{\longrightarrow}\infty \st \beta_{k_i}\tr(\Sigma_{\geqk_i})\underset{i\to\infty}{\longrightarrow}0$.
\end{assumption}
Because Mercer kernels are trace class, one always has $\tr(\Sigma_{\geqk_i})\underset{i\to\infty}{\longrightarrow}0$. As such, \assref{assumption:good_beta} simply states that for infinitely many choices of $k\in\N$,  $\beta_k$ does not increase too quickly. This is of course satisfied by the previous examples of kernels with $\beta_k=\bigo_k(1)$.

\section{Eigenvalues of Kernel Matrices}\label{sec:eigenvalues}
Since the KRR solution can be written as in \eqref{eq:krr_solution}, understanding it requires understanding the structure of the empirical kernel matrix $\krmat$. In particular, we will need to provide tight bounds on its eigenvalues. For a fixed $k\in\N$, it is known that $\mu_k\left(\frac{1}{n}\krmat\right)$ should tend to $\lambda_k$ as $n\to\infty$. In fact, there are bounds of the form $\abs{\mu_k\left(\frac{1}{n}\krmat\right) - \lambda_k}= \bigo \left(\frac{\tr(\Sigma)}{\sqrt{n}}\right)$ \citep{rosasco2010learning}. Unfortunately, these bounds are the same for all $1\leq k\leq n$, and since usually $\lambda_k=o\left(\frac{1}{k}\right)$, for most of the eigenvalues of $\frac{1}{n}\krmat$, the $\bigo \left(\frac{\tr(\Sigma)}{\sqrt{n}}\right)$ approximation error is much larger than the eigenvalues themselves, leading to the very weak bound of $0\leq\mu_k\left(\frac{1}{n}\krmat\right)\leq \bigo \left(\frac{\tr(\Sigma)}{\sqrt{n}}\right)$. This is insufficient for multiple reasons. First, the expected decay of eigenvalues in the kernel matrix is not captured. Second, tighter lower bounds are often necessary to ensure the kernel matrix is positive definite and well-conditioned. Control of the smallest eigenvalue is a common working assumption in the NTK literature \citep{du2019gradient, arora2019fine, hu2022universality} and determines the convergence rate of gradient descent with the corresponding neural network \citep{geifman2023controlling}.

We address these issues by providing \emph{relative} perturbation bounds. The general approach is to decompose the kernel matrix as $\krmat = \krmat_{\leqk} + \krmat_{\geqk}$ (for $k\in\N$), where the eigenvalues of the "low dimensional" part $\krmat_{\leqk}$ should concentrate well, and the "high dimensional" part $\krmat_{\geqk}$ should approximately be $\tilde{\gamma} I$ for some $\tilde{\gamma}>0$.

\begin{restatable}{theorem}{kereigenvalues}\label{thm:ker_eigenvalues}
    Suppose \assref{assumption:good_beta} holds, and that the eigenvalues of $\Sigma$ are given in non-increasing order $\lambda_1\geq \lambda_2 \geq \ldots$. There exist some absolute constants $c,C,c_1, c_2>0$ s.t for any $k\leq k' \in [n]$ and $\delta>0$, it holds w.p at least $1- \delta - 4 \frac{r_k}{k^4}\exp\left(-\frac{c}{\beta_k}\frac{n}{r_k}\right)-2\exp(-\frac{c}{\beta_k}\max\left(\frac{n}{k},\log(k)\right))$ that:
    \begin{align}\label{eq:ker_eigen_upper}
        \mu_k\left(\frac{1}{n}\krmat\right) \leq c_2\beta_k\Biggl(\left(1+\frac{k\log(k)}{n}\right)\lambda_{k} + \log(k+1)\frac{\tr(\Sigma_{\geqk})}{n}\Biggr),
    \end{align}
    and
    \begin{align}
        \mu_k\left(\frac{1}{n}\krmat\right) \geq c_1\mathbb{I}_{k,n}\lambda_{k} + \alpha_k\left(1-\frac{1}{\delta}\sqrt{\frac{n^2}{R_{k'}}}\right) 
        \frac{\tr\left(\Sigma_{\geqk'}\right)}{n},
    \end{align}
    where $\mathbb{I}_{k,n}=\begin{cases}
			1, & \text{if } C\beta_kk\log(k)\leq n\\
            0, & \text{otherwise}
		 \end{cases}$.
\end{restatable}

Informally, the theorem shows that one should think of the kernel matrix as $\krmat \approx \krmat_{\leqk} + \tilde{\gamma} I$ where $\tilde{\gamma}>0$ is some value which is larger the "flatter" the eigenvalue decay of $\Sigma$, and $\mu_i\left(\frac{1}{n}\krmat_{\leqk}\right)\approx \lambda_k$. More specifically, $n$ samples suffice for approximating the eigenvalues of the top $\bigo_n\left(\frac{n}{\log(n)}\right)$ features. For the largest eigenvalues of the kernel matrix, $\frac{\tr(\Sigma_{\geqk})}{n}$ should be small, and thus $\mu_k\left(\frac{1}{n}\krmat\right) \approx \lambda_k$. By contrast, for the smaller eigenvalues of the kernel matrix where $k = \omega_n\left(\frac{n}{\log(n)}\right)$, one instead has to turn towards the self-regularization induced by the $\geqk$ features. If the eigenvalues decay sufficiently slowly, one should be able to pick $k'$ so that $R_{k'} > \frac{n^2}{\delta^2}$ and $\frac{\tr\left(\Sigma_{\geqk'}\right)}{n} \approx \frac{\tr\left(\Sigma_{\geqn}\right)}{n} \approx \frac{\tr\left(\Sigma_{\geqk}\right)}{n}$. This implies that the smaller eigenvalues of the kernel matrix can be bounded as $\mu_k\left(\frac{1}{n}\krmat\right)\gtrsim \frac{\tr\left(\Sigma_{\geqn}\right)}{n}$. 

As an example, suppose $\lambda_i = \Theta\left(\frac{1}{i\log^{1+a}(i)}\right)$ for some $a>0$ and $\alpha_k, \beta_k = \Theta(1)$ (a condition satisfied by many common kernels, see \appref{appendix:example_kernels}). Then taking $k':=k'(n):=n^2$, one can easily calculate that $R_{k'}\geq \Omega(n^2\log(n))$ and $\frac{\tr\left(\Sigma_{\geqk'}\right)}{n} = \Theta\left(\frac{1}{n\log^a(k')}\right) = \Theta\left(\frac{1}{n\log^a(n)}\right)$. As a result, letting $\tilde{\gamma}_n:=\frac{1}{n\log^a(n)}$, \thmref{thm:ker_eigenvalues} implies that for any $k\in[n]$ one has that $\mu_k\left(\frac{1}{n}\krmat\right) \geq \Omega\left(\mathbb{I}_{k,n} \lambda_k + \tilde{\gamma}_n\right)$. In particular, the smallest eigenvalues can be lower bounded as $\mu_n\left(\frac{1}{n}\krmat\right) \geq \Omega\left(\frac{1}{n\log^a(n)}\right) \gg \lambda_n$. This result is at first surprising, as the classical intuition arising from works discussed earlier which bound $\abs{\mu_k\left(\frac{1}{n}\krmat\right) - \lambda_k}$ would suggest that $\mu_n\left(\frac{1}{n}\krmat\right) \approx \lambda_n$. One can analogously obtain a matching upper bound up to a $\log(k)$ factor.

The parameter $\tilde{\gamma}$ in the above example plays an identical role in KRR as the actual regularization term $\reg$. As such, the kernel actually provides its own regularization, arising from the high dimensionality of the features and the flatness of the eigenvalues. We call this \emph{self-induced regularization}, and it has two significant implications. First, it can be used to derive good bounds on the smallest eigenvalue of a kernel matrix, which as already mentioned, is critical for many applications, and will be used extensively to derive new KRR bounds in the following sections. Second, it can (quite surprisingly) cause the eigenvalues of the kernel matrix to decay at a significantly different rate than $\lambda_k$. In particular, the spectrum of $\mu_k\left(\frac{1}{n}\krmat\right)$ concentrates around $\lambda_k + \tilde{\gamma}$ for all $k\in[n]$.

\section{Excess Risk of Kernel Regression}\label{sec:krr}
We now return to bounding the bias and variance of KRR as given by \eqref{eq:bias_variance}. The strategy will be to pick some $k\leq n$, and treat the $\leqk$ and $\geqk$ components separately. By the previous section, we expect that $\krmat_{\geqk}\approx \tilde{\gamma} I$ and this will serve as a regularization term for KRR. We quantify this by what we call the \emph{concentration coefficient}
\begin{align}\label{eq:concentration}
    \err_{k,n} := \frac{\norm{\Sigma_{\geqk}} + \mu_1\left(\frac{1}{n}\krmat_{\geqk}\right) + \reg}{\mu_n\left(\frac{1}{n}\krmat_{\geqk}\right) + \reg}.
\end{align}
Because $\mu_1\left(\frac{1}{n}\krmat_{\geqk}\right) = \norm{\hat{\Sigma}_{\geqk}}$ where $\hat{\Sigma}_{\geqk}$ is the empirical covariance matrix and $\E\left[\hat{\Sigma}_{\geqk}\right]=\Sigma_{\geqk}$, one should expect that any upper bound on $\mu_1\left(\frac{1}{n}\krmat_{\geqk}\right)$ should be larger than $\norm{\Sigma_{\geqk}}$. As a result, the $\norm{\Sigma_{\geqk}}$ term practically affects $\rho_{k,n}$ by at most a factor of $2$. We only include this term for technical simplicity within the proofs. Now, if for some $k$, one shows that $\mu_1\left(\frac{1}{n}\krmat_{\geqk}\right)\approx \mu_n\left(\frac{1}{n}\krmat_{\geqk}\right)$ then the concentration coefficient $\err_{k,n}$ can be bounded as $\Theta(1)$. As we shall soon show, in such a case, it will follow that the bias and variance can be well bounded. Although our theory from the previous section provides a bound for $\err_{k,n}$, we make the role of $\err_{k,n}$ explicit in the bias and variance bounds. This is because tighter bounds on $\err_{k,n}$ may be available when there is additional information on the structure of the kernel.

\begin{restatable}{theorem}{boundgen} \label{thm:bound_gen}
Let $k\in\N$ and let $\err_{k,n}$ be as defined in \eqref{eq:concentration}. There exists some absolute constants $c,c',C_1,C_2>0$ s.t if $c\beta_kk\log(k)\leq n$, then for every $\delta>0$, it holds w.p at least $1-\delta - 16\exp\left(-\frac{c'}{\beta_k^2}\frac{n}{k}\right)$ that both the variance and bias can be upper bounded as:
\begin{align}\label{eq:bound_var}
    V \leq & C_1\err_{k,n}^2\sigma_\epsilon^2 \left(\frac{k}{n} +  \min\left(\frac{r_k\left(\Sigma^2\right)}{n}, \frac{n}{\alpha_k^2R_k(\Sigma)}\right)\right).
\end{align}
\begin{align}\label{eq:bound_bias}
        B \leq  C_2 \err_{k,n}^3\left(\frac{1}{\delta}\norm{\theta^*_\geqk}_{\Sigma_\geqk}^2 + \norm{\theta_\leqk^*}_{\Sigma_\leqk^{-1}}^2 \left(\reg + \frac{ \beta_k\tr\left(\Sigma_{\geqk}\right)}{n}\right)^2\right).
\end{align}
\end{restatable}

Several comments are in order. First, the optimal choice of $k$ should depend on the concentration coefficient $\err_{k,n}$, and the eigenvalues $\lambda_i$ of the kernel. Given these, one can determine an asymptotically optimal $k$ as a function of $n$. One would typically want to take $k$ to be as small as possible, while still ensuring $\rho_{k,n}\approx1$. Second, we do not assume here that the eigenvalues $\lambda_i$ are ordered. This is important because for certain kernels, ordering the eigenvalues is actually quite difficult, for example with NTKs corresponding to popular convolutional architectures \citep{barzilai2022kernel}. This flexibility will be critical for our analysis in the following section involving dot product kernels. Finally, a control of $\alpha_k$ is not required to obtain bounds for the bias and variance, and is present only in \eqref{eq:bound_var} via the term $\min\left(\frac{r_k\left(\Sigma^2\right)}{n}, \frac{n}{\alpha_k^2R_k(\Sigma)}\right)$. Under a slight abuse of notation, even when $\alpha=0$, this term is at most $\frac{r_k\left(\Sigma^2\right)}{n}$. As we shall later show in \thmref{thm:fixed_dimensional}, under sufficient regularization, our bounds on the excess risk will not depend on $\alpha_k$.

We also note that in the simple case of finite-dimensional linear regression (where $\phi(\x)=\x$) with zero mean and sub-Gaussian $\psi(\x)=\Sigma^{-\nicefrac{1}{2}}\x$, our bounds provide a significant generalization of those of \citet{tsigler2023benign}[Theorem 1]. Specifically, they derived similar bounds for a specific $k$ which is hard to determine, under the explicit assumption that the condition number of $\frac{1}{n}\krmat_{\geqk} + \reg I$ (similar to $\rho_{k,n}$) is bounded by some constant. Their results only hold for $0$-mean, sub-Gaussian, and finite-dimensional $\psi_i$, and hence are not applicable for many common kernels. The explicit dependence on $\rho_{k,n}$, as well as the ability to choose $k$ freely, will play an important role in the proofs of \thmref{thm:highdim} and \thmref{thm:min_norm_poly} in the next sections. Nevertheless, when all of their assumptions are satisfied, including that the condition number of $\frac{1}{n}\krmat_{\geqk} + \reg I$ is constant, our bound precisely recovers theirs. Because they showed that their bounds are sharp up to a multiplicative constant, we also obtain that under sufficient conditions, the upper bounds in \thmref{thm:bound_gen} are also sharp. 

\begin{figure}[tb]
     \centering
     \begin{subfigure}[b]{0.49\textwidth}
         \centering
         \includegraphics[width=\textwidth]{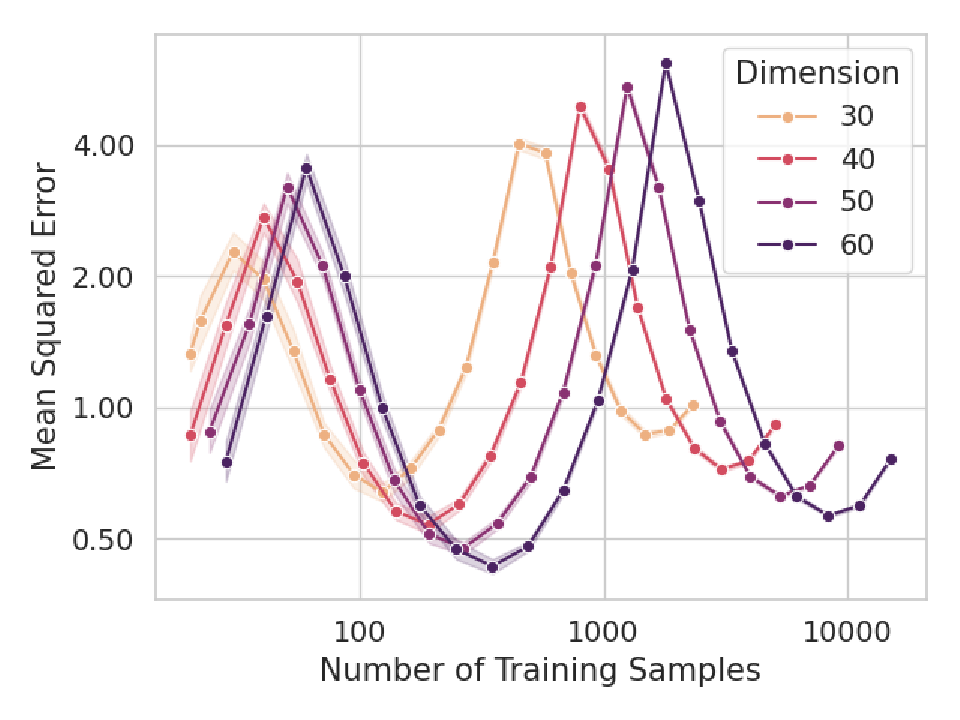}
     \end{subfigure}
     \hfill
     \begin{subfigure}[b]{0.49\textwidth}
         \centering
         \includegraphics[width=\textwidth]{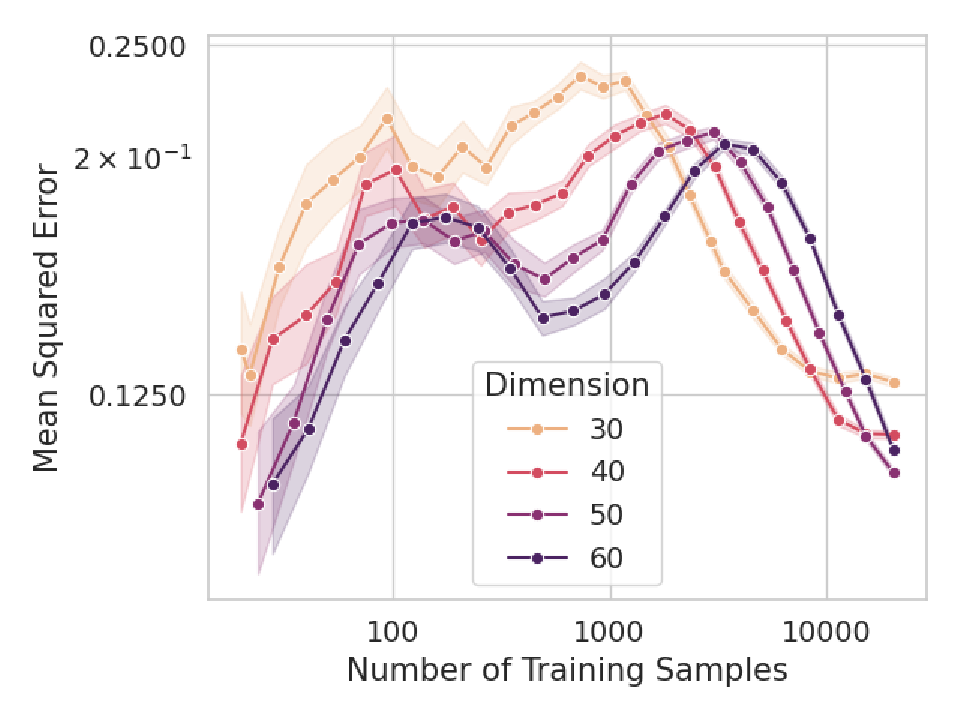}
     \end{subfigure}
        \caption{Variance of unregularized Kernel Regression, measured by the MSE for learning a constant $0$ function with noise $\epsilon_i \sim \Ncal(0, 1)$ and inputs uniformly in $\Sphere^{d-1}$ ($\log$-$\log$ scale). Left: Polynomial kernel $K(\x,\x')=(1+\frac{1}{d}\langle \x, \x'\rangle) ^ 3$; Right: NTK corresponding to a $3$-layer fully-connected network (see \appref{appendix:experiments}). As the input dimension grows, the multiple descent phenomenon becomes more pronounced, and the MSE at the "valleys" decreases. The shaded region denotes $95\%$ confidence over $50$ trials with $2500$ test samples each.}
        \label{fig:multiple_descent}
\end{figure}

\section{Applications}\label{sec:applications}
\subsection{Benign Overfitting in High Dimensions} \label{section:high-dim}
In order to capture high-dimensional phenomena that likely play a major role in the success of neural networks, it is common to analyze KRR in a high-dimensional setting. Specifically, where $n,d$ both tend towards infinity, with the ratio $\frac{n}{d^{\tau}}=\Theta(1)$ fixed for some $\tau>0$. In this chapter, we consider an important class of kernels known as \emph{dot-product kernels}. A kernel $K$ is called a dot product kernel if $K(\x,\x')=h(\x^\top \x')$ for some function $h$. One typically has to impose restriction on $h$ for $K$ to be a valid kernel, and as such, we follow the standard assumption that $h$ has a Taylor expansion of the form $h(t)=\sum_{i=0}^\infty a_i t^i$ with $a_i\geq 0$ \citep{azevedo2015eigenvalues, scetbon2021spectral}. We will currently restrict ourselves to $\Sphere^{d-1}$ (and thus $h:[-1,1]\to \R$) under the uniform distribution. Examples of dot-product kernels on $\Sphere^{d-1}$ include NTKs and GPKs of fully-connected networks and fully-connected-ResNets, Laplace kernels, Gaussian (RBF) kernels, and polynomial kernels \citep{smola2000regularization, minh2006mercer, bietti2020deep, chen2020deep}. For any $d\geq 3$, dot-product kernels with inputs uniformly distributed on $\Sphere^{d-1}$ have known Mercer decompositions given by
\begin{align}\label{eq:mercer-dot}
    K(\x,\x') = \sum_{\ell=0}^\infty \frac{\hat{\sigma}_\ell}{N(d,\ell)} \sum_{m=1}^{N(d,\ell)}Y_{\ell, m}(\x)Y_{\ell, m}(\x'),
\end{align}
where the eigenfunctions $Y_{\ell, m}$ are the $m$'th spherical harmonic of degree (or frequency) $\ell$, $N(d,\ell)=\frac{2\ell+d-2}{\ell} \binom{\ell+d-3}{d-2}$ is the number of harmonics of each degree, and $\sigma_{\ell}:=\frac{\hat{\sigma}_\ell}{N(d,\ell)}$ are the eigenvalues \citep{smola2000regularization}. Each spherical harmonic can be defined via restrictions of homogeneous polynomials to the unit sphere, with the degree (or frequency) of the spherical harmonic corresponding to the degree of said polynomials. We defer a background on dot-product kernels and more involved explanations to \appref{appendix:dot-product}. We now show that in the high-dimensional regime, any dot product kernel is capable of benign overfitting, i.e achieving an excess risk that approaches zero as $n\to\infty$, without regularization and despite the presence of noise.

\begin{restatable}{theorem}{highdim}\label{thm:highdim}
    Suppose that as $n,d\to\infty$, $\frac{d^{\tau}}{n}=\Theta_{n,d}\left(1\right)$ for some $\tau\in (0,\infty)\setminus\N$. Let $\mu$ be the uniform distribution over $\Sphere^{d-1}$, $f^*\in L_{\mu}^2(\Sphere^{d-1})$ a target function, and $K$ be a dot-product kernel given by \eqref{eq:mercer-dot} s.t $\hat{\sigma}_{\lfloor \tau \rfloor}>0$ and $\exists \ell> \lfloor 2\tau \rfloor$ with $\hat{\sigma}_{\ell}\geq 0$ (e.g NTK, Laplace, or RBF). 
    Then for the min norm solution defined in \eqref{eq:min_norm_solution} (given when $\reg\to0$), for any $\delta>0$ it holds w.p at least $1 - \delta - o_d\left(\frac{1}{d}\right)$ that
    \[
    V \leq \sigma_{\epsilon}^2 \cdot \bigo_{n,d}\left(\frac{1}{d^{\tau - \lfloor\tau \rfloor}} + \frac{1}{d^{\lfloor\tau\rfloor + 1 - \tau}}\right).
    \]
    \[
    B \leq \frac{1}{\delta}\bigo_{n,d}\left(\norm{\theta^*_{> N_d}}_{\Sigma_{> N_d}}^2\right) + \norm{\theta^*_{\leq N_d}}_{\infty}^2 \left(\underset{\ell \leq \lfloor \tau \rfloor \st \hat{\sigma}_{\ell}\neq 0}{\max} ~ \frac{1}{\hat{\sigma}_{\ell}}\right) \cdot \bigo_{n,d}\left(\frac{1}{d^{2(\tau - \lfloor\tau \rfloor)}}\right). 
    \]
    Where $N_d=\Theta_{n,d}\left(d^{\lfloor\tau \rfloor}\right)$ denotes the number of spherical harmonics of degree at most $\lfloor\tau \rfloor$ with non-zero eigenvalues, and $\bigo_{n,d}\left(\norm{\theta^*_{> N_d}}_{\Sigma_{> N_d}}^2\right) \leq \bigo_{n,d}\left(\norm{\theta^*_{> N_d}}_{\infty}^2\right)$.
\end{restatable}

Simply put, the variance decays to $0$, and the bias approaches $\bigo\left(\norm{\theta^*_{>N_d}}_{\infty}^2\right)$ for $N_d\approx d^{\lfloor \tau \rfloor}$. More specifically, the rate of decay for the variance depends on $\tau$, with the fastest decay occurring when $\tau = z+\frac{1}{2}$ for some $z\in \N$, and slowest when $\tau \approx z$. This highlights the multiple descent behavior of kernel ridge regression as discussed in \citet{liang2020just, xiao2022precise}. For the bias, $\norm{\theta^*_{> N_d}}_{\Sigma_{> N_d}}^2$ is the $L^2_{\mu}$ norm of the projection of $f^*$ onto the spherical harmonics of degree at least $\lceil \tau \rceil$, and $\norm{\theta^*_{> N^d}}_{\infty}^2$ is the maximal projection. The $\underset{\ell \leq \lfloor \tau \rfloor \st \hat{\sigma}_{\ell}\neq 0}{\max} ~ \frac{1}{\hat{\sigma}_{\ell}}$ term will typically be $\bigo_{n,d}(1)$ because often times $\hat{\sigma}_{\ell} = \Omega_{n,d}(1)$. For example, for the NTK, one has an even stronger statement, $\underset{\ell \leq \lfloor \tau \rfloor \st \hat{\sigma}_{\ell}\neq 0}{\max} ~ \frac{1}{\hat{\sigma}_{\ell}} = \bigo_{n,d}(\frac{1}{d})$ \citep{cao2019towards}[Theorem 4.3]. Thus, whether KRR achieves benign overfitting or not depends on the spectral decomposition of the target function. If $\theta^*$ consists of frequencies of at most $\lfloor \tau \rfloor$, then $\norm{\theta^*_{>N_d}}_{\infty}^2=0$ and thus both the bias and variance tend towards zero, implying benign overfitting. The variance for high-dimensional regression is demonstrated in \figref{fig:multiple_descent} for the NTK and polynomial kernel.

The key to this result is that the repeated eigenvalues lead to large effective ranks $r_k$ and $R_k$, allowing one to take $k=N_d$ (where $\frac{N_d}{n}=\frac{1}{d^{\tau - \lfloor \tau \rfloor}}$) with concentration coefficient $\err_{k,n}=\Theta(1)$. We highlight the fact that there is nothing specific to dot product kernels, and using \thmref{thm:bound_gen}, a similar result can be derived for any kernel with $\rho_{k,n} =\Theta(1)$ for $k\ll n$. The assumption that $\hat{\sigma}_{\lfloor \tau \rfloor}>0$ and $\exists \ell> \lfloor 2\tau \rfloor$ with $\hat{\sigma}_{\ell}\geq 0$ is only made for simplicity to avoid degeneracies via convoluted examples involving 0 eigenvalues. We make the role of this assumption clear within the proof, as it can easily be modified. For example, one can obtain similar results when the $0$ eigenvalues are the odd frequencies as in an NTK without bias \citep{ronen2019convergence, bietti2020deep}

Our results can naturally be extended to other domains and distributions. \citet{li2023statistical}[Corollary D.2, Lemma D.4] show that the eigenvalues only change by multiplicative constants under suitable change of measures or diffeomorphisms ("smooth" change of domains). One can also exploit the specific structure of certain kernels. For example, NTK kernels and homogeneous polynomial kernels are zonal, meaning that  $K(\x,\x')=\norm{\x}\norm{\x'}K\left(\frac{\x}{\norm{\x}}, \frac{\x'}{\norm{\x'}}\right)$, so results from $\Sphere^{d-1}$ can easily generalize to $\R^d$.

Perhaps the works that provide the results most similar  to \thmref{thm:highdim} are the excellent papers of \citet{liang2020just, mei2022generalization, xiao2022precise}.  By comparison, \citet{xiao2022precise}[Corollary 2] do not provide convergence rates, but rather show that the excess risk approaches $\norm{\theta^*_{> N_d}}_{\Sigma_{> N_d}}^2 + o_d(1)$ as $n,d\to\infty$.  Furthermore, they assumed that $\hat{\sigma}_{\ell}$ are $\Theta_d(1)$ independent of $d$, a condition which is typically not satisfied, e.g. in an NTK. \citet{mei2022generalization}[Theorem 4] when combined with a "spectral gap condition" (which would also require that $\hat{\sigma}_{\ell}$ are $\Theta_d(1)$) also implies a bound of the form  $\norm{\theta^*_{> N_d}}_{\Sigma_{> N_d}}^2 + o_d(1)$. Without this problematic spectral gap assumption, it is unclear what their bound implies.  They also impose other strict assumptions, which do not hold for broader domains. For example, they assume that for any $\x_i$, $\frac{\norm{\phi_{> N_d}(\x_i)}^2}{\tr\left(\Sigma_{> N_d}\right)} = 1 \pm o_d(1)$. For zonal kernels such as the NTK, this typically will not hold unless all inputs have roughly the same norm. By contrast, our mild assumptions imply that the same results hold in $\R^d$ as discussed above. The results of \citet{liang2020just}[Theorem 3] are limited to target functions in the RKHS, with a bound that is the same for all $\theta^*$. This is critical since the structure of $\theta^*$ is precisely what allows us to characterize when benign overfitting occurs. 

Overall, our results are the first to clearly characterize benign overfitting for common kernels, such as NTK. 

\subsection{Nearly Tempered Overfitting in Fixed Dimensions}

\begin{figure}[tb]
    \centering
    \includegraphics[width=0.5\textwidth]{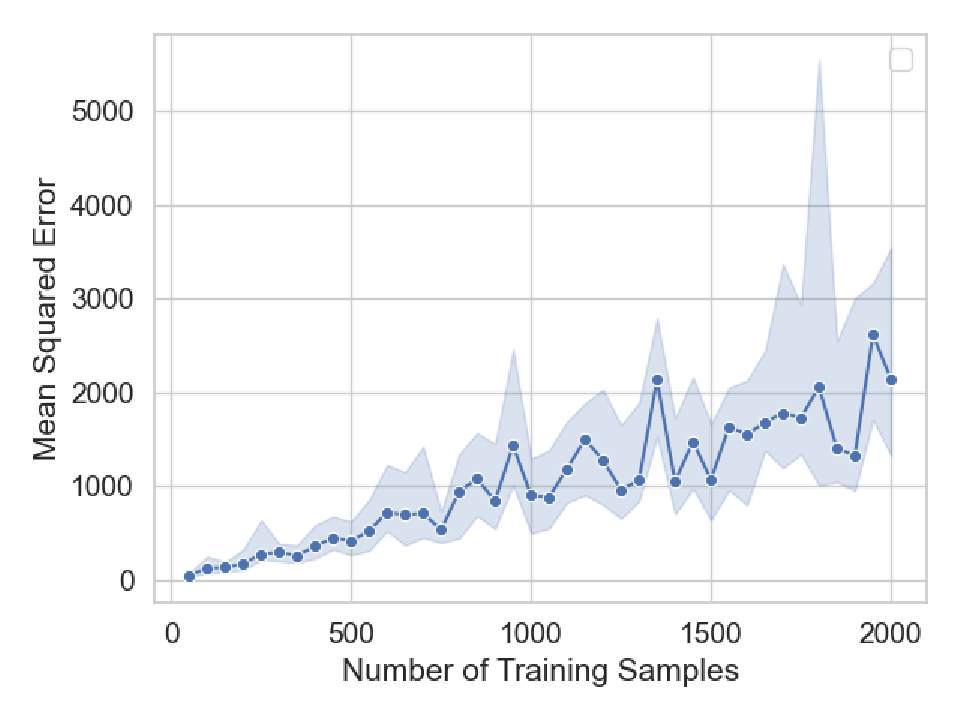}
    \caption{Apparently diverging variance in low dimensions, for a GPK corresponding to a 3-layer fully-connected network (see \appref{appendix:experiments}) with inputs distributed uniformly on the unit disk $\{x\in\R^2 : \norm{x}\leq 1\}$ and noise $\epsilon\sim \Ncal(0, 1)$. The solid line denotes the median variance (and not mean, due to extreme values), and the shaded region denotes $95\%$ confidence over $100$ trials with $5000$ test samples each. This suggests that previous works that inferred $V\leq \bigo(1)$ for kernels with polynomially decaying eigenvalues may be overly optimistic.}
    \label{fig:low_dimensional}
\end{figure}

We now shift our attention to the fixed dimensional regime. We focus on polynomially decaying eigenvalues, encompassing NTKs and GPKs of common fully-connected architectures \citep{bietti2020deep}, convolutional and residual architectures \citep{geifman2022spectral, barzilai2022kernel} as well as the Laplace kernel \citep{chen2020deep}.

For such kernels, various works show lower bounds of the form $\Omega(1)$ for the excess risk for min-norm interpolation \citep{rakhlin2019consistency, haas2023mind}. Recently \citet{mallinar2022benign} distinguished between the regimes where the risk explodes to $\infty$ (called catastrophic overfitting) vs when the risk remains bounded (called tempered overfitting). The two regimes are significantly different since when the noise is small, kernel regression can still achieve a low risk despite tempered overfitting. Using our tools, we show that when $\lambda_i \approx i^{-1-a}$ for small $a>0$, such kernels are \emph{nearly tempered}, meaning that the bias goes to $0$, and the variance cannot diverge too quickly.

\begin{restatable}{theorem}{minnormpoly}\label{thm:min_norm_poly}
    Let $K$ be a kernel with polynomially decaying eigenvalues $\lambda_i=\Theta_{i,n}(i^{-1-a})$ for some $a>0$ and assume that $\alpha_k, \beta_k = \Theta_k(1)$. Then for the min norm solution defined in \eqref{eq:min_norm_solution} (given when $\reg\to0$), for any $\delta>0$ it holds w.p at least $1-\delta - \bigo_n\left(\frac{1}{\log(n)}\right)$ that
    \begin{align*}
        V \leq \sigma_{\epsilon}^2 \tilde{\bigo}_n\left(n^{2a}\right).
    \end{align*}
    Moreover, if $\theta_i^*=\bigo_i\left(\frac{1}{i^r}\right)$ where $r>a$ then under the same probability it also holds that 
    \begin{align*}
        B \leq \frac{1}{\delta} \tilde{\bigo}_n\left(\frac{1}{n^{\min\left(2(r-a), 2-a\right)}}\right).
    \end{align*}
\end{restatable}

When $a\to 0$, the bound for the variance approaches $\polylog(n)$, and the bound for the bias is nearly $\tilde{\bigo}\left(\frac{1}{n^{2r}}\right)$. For the popular NTK of a fully-connected network and a Laplace kernel, $\lambda_i=\Theta\left(i^{-1-\frac{1}{d-1}}\right)$ \citep{chen2020deep}, indicating that $a=\frac{1}{d-1}$. For these kernels, the variance bound becomes $\tilde{\bigo}\left(n^{\frac{2}{d-1}}\right)$. In fact, when $d\gtrsim \log(n)$ it holds that $n^{\frac{2}{d-1}}\lesssim \polylog(n)$. So, when the noise is small, one can expect the excess risk to also be relatively small. The condition on the decay of $\theta^*$ is fairly mild, as for any realizable $f^*$ (i.e $f^*\in\Hcal$) it holds that $\norm{\theta^*}_2<\infty$ and thus, under the conditions of the theorem, $r> 1$ and $B< \tilde{\bigo}(\frac{1}{n^{2-2a}})$.

As far as we know, this is the first rigorous upper bound for the excess risk of the min-norm interpolator in the fixed dimensional setting for generic kernels. Previous bounds were either based on a Gaussian feature assumption or non-rigorous analysis \citep{cui2021generalization, mallinar2022benign} and gave $\bigo\left(n^{-\min(2r+a, 2(1+a))} \right)$ and $\sigma_{\epsilon}^2\cdot\bigo(1)$ bounds for the bias and variance respectively. In \figref{fig:low_dimensional}, we provide a simple example of a common kernel that does not appear to adhere to their bounds (a GPK corresponding to a 3-layer fully connected network with inputs uniformly on the unit disk). The difference between our bounds and theirs is not a limitation of our work but rather due to their strong Gaussian feature assumption and can be quantified by the concentration coefficient $\err_{k,n}$. Without any special assumptions, we showed that for $k\approx \frac{n}{\log(n)}$, $\err_{k,n}=\bigo\left(n^{a}\polylog(n)\right)$. If one is willing to make stronger assumptions on the features which may not hold in practice (such as Gaussianity) so that $\err_{k,n}=\Theta(1)$, our bias and variance bounds would improve to $\tilde{\bigo}\left(n^{-\min(2r+a, 2(1+a))}\right)$ and $\sigma_{\epsilon}^2\tilde{\bigo}(1)$ respectively, matching their bound up to a $\polylog$ factor. When $a\to 0$, the difference is of course very small, implying that one obtains nearly tempered overfitting in the fixed dimensional regime. Unfortunately, common kernels do not have Gaussian features in practice and may suffer from poor concentration in the fixed $d$ regime. Thus, a $\polylog$ factor in the bounds is likely inevitable. This is the reason for the observation in \figref{fig:low_dimensional}, showing that upper bounds that assume Gaussian features may be over-optimistic for common kernels. 

\subsection{Regularized Regression}
A major benefit to our approach is that we can provide bounds for both the regularized and unregularized cases with the same tools. We can thus derive bounds for the classical setup where the regularization $\reg$ is relatively large. 

\begin{restatable}{theorem}{strongreg}\label{thm:fixed_dimensional}
    Let $K$ be a kernel with polynomially decaying eigenvalues $\lambda_i=\Theta_{i,n}(i^{-1-a})$ for some $a>0$, and assume that $\beta_k = \bigo_k(1)$. Further, suppose that the regularization parameter satisfies $\reg = \Theta_n(n^{-1-b})$ for $b\in (-1,a)$. Then for any $\delta>0$, it holds w.p at least $1- \delta - o_n(\frac{1}{n})$ that
    \[
    V \leq \sigma_{\epsilon}^2 \cdot \bigo_n\left(\frac{1}{n^{\frac{a-b}{1+a}}}\right),
    \]
    and if $\theta_i^* = \Theta_{i,n}\left(i^{-r}\right)$ for some $r\in\R$ s.t $\norm{\Sigma^{\nicefrac{1}{2}}\theta^*}_2<\infty$ (necessary for $f^*\in L^2_{\mu}(\Xcal)$), then under the same probability it also holds that 
    \[
     B \leq \frac{1}{\delta} \cdot \bigo_{n}\left(\frac{1}{n^{(1+b)\min\left(\frac{(2r+a)}{1+a}, 2\right)}}\right),
    \]
    where the $\bigo$ is weakened to $\tilde{\bigo}$ if $r = 1+\frac{a}{2}$.
\end{restatable}

The conditions of \thmref{thm:fixed_dimensional} are very mild, and do not require any control of $\alpha_k$. In particular, the same kernels mentioned in the previous chapter all satisfy the assumptions here. Regarding the role of the regularization decay, as $b$ decreases, the regularization is strengthened. One can observe a bias-variance tradeoff, where the variance bound improves with increased regularization, and the bias bound worsens. Regardless, one always has that the excess risk tends to $0$ as $n\to\infty$. The choice of polynomial decay was arbitrary, and bounds for other decays can easily be obtained by modifying the proof.

The result recovers those of \citet{cui2021generalization} who worked under the heavy Gaussian feature assumption, and \citet{li2023asymptotic} who worked under a H\"{o}lder continuity assumption on the kernel as well as an assumption relating to what they called an embedding index. \citet{caponnetto2007optimal} only provide upper bounds for the optimal $\reg$, and do not decompose into bias and variance.

\section{Implications for Neural Networks}\label{sec:neural_nets}
Our mild assumptions and general setting allow us to apply these results to a wide range of neural networks. Under suitable initialization and learning rate, gradient decent with sufficiently wide neural networks is equivalent to kernel regression with the NTK \citep{jacot2018neural, lee2019wide, yang2021tensor}. Specifically, for a neural network $f(\x,\theta)$, one can typically bound its distance from its first order Taylor approximation $f^{\text{lin}}(\x,\theta)$ at time $t$ of gradient flow as $\sup_{t\geq 0}\abs{f(\x,\theta_t) - f^{\text{lin}}(\x,\theta_t)} \leq O\left(\frac{1}{\sqrt{\text{width}}}\right)$ \citep{lee2019wide, bowman2022spectral}. Furthermore, training $f^{\text{lin}}(\x,\theta)$ for time $t$ is roughly equivalent to kernel regression with regularization $\reg = \frac{1}{t}$ \citep{ali2019continuous}. By combining the two, one can easily bound the difference in generalization errors between neural networks trained for time $t$ and kernel regression with the NTK and regularization $\reg = \frac{1}{t}$.
So by \thmref{thm:fixed_dimensional}, if the eigenvalues of the NTK decay as $\lambda_i=\Theta_{i,n}(\frac{1}{i^{-1-a}})$ and the target function satisfies $\theta_i^* = \Theta_{i,n}\left(i^{-r}\right)$, then as the width of the corresponding network tends towards infinity, the bias and variance after training for time $t:=\Theta_n\left(n^{s}\right)$ with gradient flow for some $s\in(0, 1+a)$ approach
    \begin{align}\label{eq:learning_rates}
        V \leq \sigma_{\epsilon}^2 \cdot \bigo_n\left(\frac{1}{n^{1-\frac{s}{1+a}}}\right), \qquad 
        B \leq \bigo_n\left(\frac{1}{n^{s\min\left(\frac{2r+a}{1+a}, 2\right)}}\right).
    \end{align}
Neural networks of various architectures exhibit polynomially decaying eigenvalues in the fixed dimensional regime, including fully-connected networks, CNNs, and ResNets \citep{bietti2020deep, geifman2022spectral, barzilai2022kernel}. For example, for fully-connected networks, $a=\frac{1}{d-1}$. Interestingly, skip connections do not affect the asymptotic rate of decay of the NTK eigenvalues \citep{barzilai2022kernel, belfer2021spectral} and as a result, ResNets obtain the same rates in \eqref{eq:learning_rates} as their non-residual counterparts (i.e if one removes the skip connection).

Similarly, the applications of \thmref{thm:min_norm_poly} and \thmref{thm:highdim} to networks that are instead trained to completion (i.e in the $t\to\infty$ limit) are immediate. In particular, one has nearly tempered overfitting in the fixed dimensional regime, and in the high dimensional regime of $\frac{d^\tau}{n}=\Theta(1)$, if $f^*$ consists of frequencies of at most $\lceil \tau \rceil$, then overfitting is benign.

\subsection*{Acknowledgments}
This research is supported in part by European Research Council (ERC) grant 754705, by the Israeli Council for Higher Education (CHE) via the Weizmann Data Science Research Center and by research grants from the Estate of Harry Schutzman and the Anita James Rosen Foundation.

\bibliography{refs}

\pagebreak
\appendix
\tableofcontents

\section{More Notations}\label{app:notations}
We introduce a few more notations for the appendix, which are not needed in the main text. We let $A_k:=\krmat_{\geqk}+n\reg I$ and $A:=\krmat+n\reg I$. Additionally, for any $k\leq k'\in\N$ we denote by $k:k'$ the $k,\ldots, k'$ indices, so that, for example, $\phi_{k:k'}(X)=(\phi_k(X), \ldots, \phi_{k'}(X))\in\R^{n\times (k'-k+1)}$.

\section{Concentration Bounds}

\begin{lemma} \label{lem:concentration1}
    Let $k\in[n]$, then each of the following holds w.p at least $1-2\exp\left(-\frac{1}{2\beta_k^2}n\right)$:
    \begin{enumerate}
        \item $\frac{1}{2}n\sum_{i\geqk} \lambda_i^2 \leq \tr\left(\phi_{\geqk}(X)\Sigma_{\geqk}\phi_{\geqk}(X)^\top \right) \leq \frac{3}{2}n\sum_{i\geqk}\lambda_i^2$

        \item $\frac{1}{2}kn \leq \tr\left(\psi_{\leqk}(X)\psi_{\leqk}(X)^\top \right) \leq \frac{3}{2}kn$.
    \end{enumerate}
\end{lemma}

\begin{proof}
    For (1), first observe that
    \begin{align*}
        & \tr\left(\phi_{\geqk}(X)\Sigma_{\geqk}\phi_{\geqk}(X)^\top \right) \\
        =& \sum_{j=1}^n \left[\phi_{\geqk}(X)\Sigma_{\geqk}\phi_{\geqk}(X)^\top \right]_{jj}
        = \sum_{j=1}^n \phi_{\geqk}(\x_j)^\top \Sigma_{\geqk}\phi_{\geqk}(\x_j) \\
        =& \sum_{j=1}^n\sum_{i\geqk} \lambda_i^2 \psi_i(\x_j)^2.
    \end{align*}

    We will now show that the conditions for Hoeffding's inequality hold. Let $v_j=\sum_{i\geqk} \lambda_i^2 \psi_i(\x_j)^2$ and $M:=\beta_k \sum_{i\geqk} \lambda_i^2$. By the definition of $\beta_k$ \eqref{def:eigen_upper}, we have that for every $j$, $0\leq v_j\leq M$. Furthermore, $\E[\sum_{j=1}^n v_j]=n\sum_{i\geqk} \lambda_i^2$ and so Hoeffding's inequality yields:
    \[
    \mathbb{P}\left(\abs{\tr\left(\phi_{\geqk}(X)\Sigma_{\geqk}\phi_{\geqk}(X)^\top \right) - n\sum_{i\geqk} \lambda_i^2} \geq t \right) \leq 2\exp\left(-\frac{2t^2}{nM^2}\right).
    \]
    Substituting $t=\frac{n}{2}\sum_{i\geqk} \lambda_i^2$, it holds that w.p at least $1-2\exp\left(-\frac{1}{2\beta_k^2}n\right)$, 
    \[
    \frac{1}{2}n\sum_{i\geqk} \lambda_i^2 \leq \tr\left(\phi_{\geqk}(X)\Sigma_{\geqk}\phi_{\geqk}(X)^\top \right) \leq \frac{3}{2}n\sum_{i\geqk} \lambda_i^2.
    \]

    For (2), the proof is analogous:
    \begin{align*}
        & \tr\left(\psi_{\leqk}(X)\psi_{\leqk}(X)^\top \right) \\
        =& \sum_{j=1}^n \left[\psi_{\leqk}(X)\psi_{\leqk}(X)^\top \right]_{jj}
        = \sum_{j=1}^n \psi_{\leqk}(\x_j)^\top \psi_{\leqk}(\x_j) \\
        =& \sum_{j=1}^n\sum_{i=1}^{k} \psi_i(\x_j) ^2 \leq \beta_kkn
    \end{align*}
    Now letting $M'=\beta_kk$ using Hoeffding as before yields
    \[
    \mathbb{P}\left(\abs{\tr\left(\psi_{\leqk}(X)\psi_{\leqk}(X)^\top \right) - kn} \geq t' \right) \leq 2\exp\left(-\frac{2t'^2}{nM'^2}\right).
    \]
    So picking $t'=\frac{nk}{2}$ we get that w.p at least $1-2\exp\left(-n\cdot \frac{1}{2\beta_k^2}\right)$
    \[
    \frac{1}{2}kn \leq \tr\left(\psi_{\leqk}(X)\psi_{\leqk}(X)^\top \right) \leq \frac{3}{2}kn.
    \]

\end{proof}

\begin{lemma} \label{lem:concentration2}
    For any $k\in [n]$ there exist some absolute constants $c',c_2>0$, s.t the following hold simultaneously w.p at least $1-2\exp\left(-\frac{c'}{\beta_k}\max\left(\frac{n}{k},\log(k)\right)\right)$
    \begin{enumerate}
        \item \label{en:lower} $\mu_k\left(\psi_{\leqk}(X)^\top\psi_{\leqk}(X)\right) \geq \max\left(\sqrt{n} - \sqrt{\frac{1}{2}\max\left(n, \beta_k\left(1+\frac{1}{c'}\right)k\log(k)\right)} ~,~ 0\right)^2$,

        \item \label{en:upper} $\mu_1\left(\psi_{\leqk}(X)^\top\psi_{\leqk}(X)\right) \leq c_2 \max\left(n, \beta_kk\log(k)\right)$.
    \end{enumerate}
    Moreover, there exists some $c>0$ s.t if $c\beta_kk\log(k)\leq n$ then w.p at least $1-2\exp\left(-\frac{c'}{\beta_k}\frac{n}{k}\right)$ and some absolute constant $c_1>0$, it holds that
    \[
    c_1n \leq \mu_k\left(\psi_{\leqk}(X)^\top\psi_{\leqk}(X)\right) \leq \mu_1\left(\psi_{\leqk}(X)^\top\psi_{\leqk}(X)\right) \leq c_2 n.
    \]
\end{lemma}
\begin{proof}
    We will bound the singular values $\sigma_i\left(\psi_{\leqk}(X)\right)$ since
    \[
    \sigma_i(\psi_{\leqk}(X))^2 = \mu_i\left(\psi_{\leqk}(X)^\top\psi_{\leqk}(X)\right).
    \]
    $\psi_{\leqk}(X)$ is an $n\times k$ matrix, whose rows $\psi_{\leqk}(\x_j)$ are independent isotropic random vectors in $\R^{k}$ (where the randomness is over the choice of $\x_j$). Furthermore, by the definition of $\beta_k$ \eqref{def:eigen_upper}, for a.s every $\x_i$, $\norm{\psi_{\leqk}(\x_i)} \leq \sqrt{\beta_kk}$. As such, from \citet{vershynin2010introduction}[Theorem 5.41], there is some absolute constant $c'>0$ s.t for every $t\geq 0$, one has that with probability at least $1-2k \exp(-2c't^2)$, 
    \begin{align*}
        \sqrt{n} - t\sqrt{\beta_kk} \leq \sigma_k(\psi_{\leqk}(X)) \leq \sigma_1(\psi_{\leqk}(X)) \leq \sqrt{n} + t\sqrt{\beta_kk}.
    \end{align*}

    Now for $t = \sqrt{\frac{1}{2\beta_k}\max\left(\frac{n}{k},\log(k)\right)+\frac{\log(k)}{2c'}}$ we get that with probability at least $1-2\exp\left(-\frac{c'}{\beta_k}\max\left(\frac{n}{k},\log(k)\right)\right)$ it holds that
    \begin{align*}
        \sigma_1(\psi_{\leqk}(X))^2
        \leq & \left(\sqrt{n} + \sqrt{\frac{1}{2}\max\left(n, k\log(k)\right) + k\log(k)\frac{\beta_k}{2c'}}\right)^2 \\
        \leq & \left(\sqrt{n} + \frac{1}{\sqrt{2}}\sqrt{n + \left(1 + \frac{\beta_k}{c'}\right)k\log(k)}\right)^2 \\
        \leq & 3n + \left(1 + \frac{\beta_k}{c'}\right)k\log(k),
    \end{align*}
    where the last equality followed from the fact that $(a+b)^2\leq2a^2+2b^2$ for any $a,b\in\R$. Because, $\beta_k \geq 1$ \eqref{eq:alpha_leq_beta}, we obtain $\sigma_1(\psi_{\leqk}(X))^2 \leq c_2 \max\left(n, \beta_kk\log(k)\right)$ for a suitable $c_2>0$, proving point \eqref{en:upper}.
    For the lower bound, we simultaneously have
    \begin{align*}
        \sigma_k(\psi_{\leqk}(X)) \geq & \sqrt{n} - \frac{1}{\sqrt{2}}\sqrt{\frac{1}{2}\max\left(n, k\log(k)\right) + k\log(k)\frac{\beta_k}{2c'}} \\
        \geq & \sqrt{n} - \sqrt{\frac{1}{2}\max\left(n, \beta_k\left(1+\frac{1}{c'}\right)k\log(k)\right)}, 
    \end{align*}
    Since the singular values are non-negative, the above implies
    \begin{align*}
        \sigma_k(\psi_{\leqk}(X))^2 \geq \max\left(\sqrt{n} - \sqrt{\frac{1}{2}\max\left(n, \beta_k\left(1+\frac{1}{c'}\right)k\log(k)\right)} ~,~ 0\right)^2.
    \end{align*}
    proving point \eqref{en:lower}.

    For the moreover part, taking $c=\left(1+\frac{1}{c'}\right)$, we now have by assumption that $\frac{n}{k}\geq c\beta_k\log(k) \geq \log(k)$ (where we used the facts that $c\geq 1$ and $\beta_k\geq 1$), the probability that \eqref{en:lower} and \eqref{en:upper} hold is in fact $1-2\exp\left(-\frac{c'}{\beta_k}\frac{n}{k}\right)$.
    
    Furthermore, plugging $c\beta_kk\log(k)\leq n$ into the lower bound \eqref{en:lower} yields
    \begin{align*}
        \mu_k\left(\psi_{\leqk}(X)^\top\psi_{\leqk}(X)\right) \geq & \max\left(\sqrt{n} - \sqrt{\frac{1}{2}\max\left(n, c\beta_kk\log(k)\right)} ~,~ 0\right)^2. \\ 
        \geq & \left(\sqrt{n} - \sqrt{\frac{n}{2}}\right)^2 = \left(1-\frac{1}{\sqrt{2}}\right)^2n.
    \end{align*}
    
    Similarly, since $\beta_kk\log(k)\leq n$ the upper bound \eqref{en:upper} becomes
    \begin{align*}
        \mu_1\left(\psi_{\leqk}(X)^\top\psi_{\leqk}(X)\right) \leq c_2n
    \end{align*}
\end{proof}

\begin{lemma} \label{lem:concentration3}
    For any $k\in [n]$ and $\delta>0$, it holds w.p at least $1-\delta$ that
    \[
    \norm{\phi_\geqk(X)\theta_\geqk^*}^2 \leq \frac{1}{\delta} n\norm{\theta_\geqk^*}_{\Sigma_\geqk}^2
    \]
\end{lemma}
\begin{proof}
    Let $v_j=\langle\phi_\geqk(\x_j), \theta_\geqk^*\rangle^2$ so that $\norm{\phi_\geqk(X)\theta_\geqk^*}^2 = \sum_{j=1}^n v_j$.
    Since $\x_j$ are independent, it holds that $v_j$ are independent random variables with mean:
    \begin{align*}
        \E[v_j] = & \E\left[\left(\sum_{i\geqk} \sqrt{\lambda_i}\psi_i(\x_j)\theta_i^*\right)^2\right] \\
        = & \sum_{i\geqk} \sum_{\geql}\sqrt{\lambda_i}\sqrt{\lambda_l}\theta_i^*\theta_l^* \underset{\delta_{il}}{\underbrace{\E_{\x_j}\left[\psi_i(\x_j)\psi_l(\x_j)\right]}} \\
        = & \sum_{i\geqk} \lambda_i (\theta_i^*)^2 = \norm{\theta^*}_{\Sigma_{\geqk}}^2.
    \end{align*}

    So by Markov's inequality:
    \[
    \mathbb{P}\left(\sum_{j=1}^n v_j \geq \frac{1}{\delta} n\norm{\theta_\geqk^*}_{\Sigma_\geqk}^2 \right) \leq \delta.
    \]

\end{proof}

\begin{lemma}\label{lem:concentration_union}
    There exists some absolute constants $c, c', c_1, c_2>0$ s.t for any $k\in\N$ with $c\beta_kk\log(k)\leq n$, it holds w.p at least $1-8\exp\left(-\frac{c'}{\beta_k^2}\frac{n}{k}\right)$ that all of the following hold simultaneously:
    \begin{enumerate}
        \item $c_1n\sum_{i\geqk} \lambda_i^2 \leq \tr\left(\phi_{\geqk}(X)\Sigma_{\geqk}\phi_{\geqk}(X)^\top \right) \leq c_2n\sum_{i\geqk} \lambda_i^2$

        \item $c_1kn \leq \tr\left(\psi_{\leqk}(X)\psi_{\leqk}(X)^\top \right) \leq c_2kn$

        \item $\mu_k\left(\psi_{\leqk}(X)^\top\psi_{\leqk}(X)\right) \geq c_1n$

        \item $\mu_1\left(\psi_{\leqk}(X)^\top\psi_{\leqk}(X)\right) \leq c_2n$
    \end{enumerate}
\end{lemma}

\begin{proof}
    By \lemref{lem:concentration1}, points (1) and (2) each hold w.p at least $1-2\exp\left(-\frac{1}{2\beta_k^2}n\right)$ so they both hold w.p at least $\left(1-2\exp\left(-\frac{1}{2\beta_k^2}n\right)\right)^2$. 

    Furthermore, the "moreover" part of \lemref{lem:concentration2} states that points (3) and (4) hold simultaneously w.p at least $1-2\exp\left(-\frac{c'}{\beta_k}\frac{n}{k}\right)$.

    Now the probability for which (1)-(4) all hold simultaneously is at least
    \begin{align*}
        & \left(1-2\exp\left(-\frac{1}{2\beta_k^2}n\right)\right)^2\left(1-2\exp\left(-\frac{c'}{\beta_k}\frac{n}{k}\right)\right) \\ 
        \geq & 1-8\exp\left(-\min\left(\frac{1}{2\beta_k^2}n, \frac{c'}{\beta_k}\frac{n}{k}\right)\right) \geq 1-8\exp\left(-\min\left(\frac{1}{2\beta_k^2}, \frac{c'}{\beta_k}\right)\frac{n}{k}\right)
    \end{align*}

    Since $\beta_k\geq 1$ \eqref{eq:alpha_leq_beta} replacing $c'$ with $\min(\frac{1}{2}, c')$ results in the desired bounds holding w.p at least $1-8\exp\left(-\frac{c'}{\beta_k^2}\frac{n}{k}\right)$.
    
\end{proof}

\section{Bounds on the Eigenvalues of Kernel Matrices - Proofs of Results in \secref{sec:eigenvalues}}\label{app:eigenvalues}
\subsection{Proof of \thmref{thm:ker_eigenvalues}}
\kereigenvalues*
\begin{proof}
    From \lemref{lem:rel_bound_eigenvals}, we have that
    \begin{align}\label{eq:rel_bound_helper}
        \lambda_{k} \mu_k\left(D_k\right) + \mu_n\left(\frac{1}{n}\krmat_{\geqk}\right) \leq \mu_k\left(\frac{1}{n}\krmat\right) \leq \lambda_{k} \mu_1\left(D_k\right) + \mu_1\left(\frac{1}{n}\krmat_{\geqk}\right),
    \end{align}
    where $D_i$ is as in the formulation of the lemma. 
    
    We bound each of the summands in the upper bound separately. From \corref{cor:ak_eigen_bound}, it holds w.p at least $1- 4 \frac{r_k}{k^4}\exp\left(-\frac{c'}{\beta_k}\frac{n}{r_k}\right)$ that for some absolute constants $c',c_2'>0$,  
    \[
    \mu_1\left(\frac{1}{n}\krmat_{\geqk}\right) \leq c_2'\left(\lambda_{k+1} + \beta_k \log(k+1)\frac{\tr\left(\Sigma_{\geqk}\right)}{n}\right).
    \]
    For the other summand, since $D_i=\frac{1}{n}\psi_{\leqk}(X)^\top\psi_{\leqk}(X)$ \lemref{lem:concentration2} states that there exists some absolute constants $c'', c_2''>0$, s.t w.p at least $1-2\exp\left(-\frac{c''}{\beta_k}\max\left(\frac{n}{k},\log(k)\right)\right)$
   \[
        \lambda_{k}\mu_1\left(D_i\right) \leq c_2'' \frac{1}{n}\max\left(n, \beta_kk\log(k)\right) \lambda_{k} \leq c_2''\beta_k\left(1+\frac{k\log(k)}{n}\right)\lambda_{k},
    \]
    where in the last inequality we used the fact that $\beta_k \geq 1$. So taking $c=\max(c',c'')$, both events hold w.p at least $1-4 \frac{r_k}{k^4}\exp\left(-\frac{c}{\beta_k}\frac{n}{r_k}\right)-2\exp(-\frac{c}{\beta_k}\max\left(\frac{n}{k},\log(k)\right))$ and the upper bound from \eqref{eq:rel_bound_helper} yields
    \[
        \mu_k\left(\frac{1}{n}\krmat\right) \leq c_2\beta_k\left(\left(1+\frac{k\log(k)}{n}\right)\lambda_{k} + \log(k+1)\frac{\tr(\Sigma_{\geqk})}{n}\right),
    \]
    for some suitable absolute constant $c_2>0$.
    The "moreover" part of this proof analogously follows from the "moreover" part of \lemref{lem:concentration2}, which states that $\mu_k(D_k)\geq c_1$ if $C\beta_kk\log(k)\leq n$, and from the lower bound of \corref{cor:ak_eigen_bound}, which holds w.p at least $1-\delta$.
\end{proof}

\subsection{Lemmas and Alternative Results for Eigenvalue Bounds}

We now provide an extension of Ostrowski's theorem to non-square matrices. Note that the case of $k\leq n$ is relatively easy. However, we also prove the case of $k>n$.
\begin{lemma}[Extension of Ostrowski's Theorem]\label{lem:ostrowski}
    Let $i,k\in\N$ satisfy $1\leq i\leq \min(k,n)$ and $D_k:=\frac{1}{n}\psi_{\leqk}(X)\psi_{\leqk}(X)^\top\in\Real^{n\times n}$. Suppose that the eigenvalues of $\Sigma$ are given in non-increasing order $\lambda_1\geq \lambda_2 \geq \ldots$ then 
    \[
    \lambda_{i+k-\min(n,k)}\mu_{\min(n,k)}(D_{k}) \leq \mu_i\left(\frac{1}{n}\krmat_{\leqk}\right) \leq \lambda_i\mu_1(D_k).
    \]
\end{lemma}
\begin{proof}
    Let $\pi_1$ denote the number of positive eigenvalues of $\frac{1}{n}\krmat_\leqk$ (where in particular $\pi_1\leq \min(n,k)$). Because the kernel can be decomposed as $\krmat_{\leqk}=\psi_{\leqk}(X)\Sigma_{\leqk}\psi_{\leqk}(X)^\top$, it follows from \citet{dancis1986quantitative}[Theorem 1.5] that for $1\leq i\leq \pi_1$, 
    \[
    \lambda_{i+k-\min(n,k)}\mu_{\min(n,k)}(D_{k}) \leq \mu_i\left(\frac{1}{n}\krmat_{\leqk}\right) \leq \lambda_i\mu_1(D_k).
    \]

    \noindent
    It remains to handle the case where $\pi_1 < i$ (where in particular this means $\pi_1<\min(n,k)$). By definition of $\pi_1$ there are some orthonormal eigenvectors of $\krmat_{\leqk}$, $v_{\pi_1+1},\ldots, v_n$ with eigenvalues $0$. Since $\Sigma \succ 0$, for each such $0$ eigenvector $v$,
    \[
    0 = \left(\psi_{\leqk}^\top(X) v\right)^\top \Sigma \left(\psi_{\leqk}^\top(X) v\right) \implies \psi_{\leqk}^\top(X) v = 0.
    \]
    In particular, $D_k$ has $v_{\pi_1+1},\ldots, v_n$ as $0$ eigenvectors and since $D_k\succeq 0$, we obtain that $\mu_{\pi_1+1}(D_k),\ldots, \mu_n(D_k) = 0$. So for $i> \pi_1$ we have
    \[
    \lambda_{i+k-\min(n,k)}\mu_{\min(n,k)} (D_k) = 0 = \mu_i\left(\frac{1}{n}\krmat_{\leqk}\right) \leq \lambda_i\mu_1(D_k).
    \]
\end{proof}

\begin{lemma} \label{lem:rel_bound_eigenvals}
Let $i,k\in\N$ satisfy $1\leq i \leq n$ and $i\leq k$, let $D_k:=\frac{1}{n}\psi_{\leqk}(X)\psi_{\leqk}(X)^\top\in\Real^{n\times n}$. that the eigenvalues of $\Sigma$ are given in non-increasing order $\lambda_1\geq \lambda_2 \geq \ldots$ then
\[
\lambda_{i+k-\min(n,k)}\mu_{\min(n,k)}(D_{k}) + \mu_n\left(\frac{1}{n}\krmat_{\geqk}\right) \leq \mu_i\left(\frac{1}{n}\krmat\right) \leq \lambda_i \mu_1\left(D_k\right) + \mu_1\left(\frac{1}{n}\phi_{\geqk}(X)\phi_{\geqk}(X)^\top\right).
\]
In particular, 
\[
\lambda_{i+k-\min(n,k)}\mu_{\min(n,k)}(D_{k}) \leq \mu_i\left(\frac{1}{n}\krmat\right) \leq \lambda_i \mu_1\left(D_k\right) + \mu_1\left(\frac{1}{n}\phi_{\geqk}(X)\phi_{\geqk}(X)^\top\right).
\]
\end{lemma}
\begin{proof}
    We can decompose $\krmat$ into the sum of two hermitian matrices by $\krmat=\krmat_{\leqk}+\krmat_{\geqk}$. By Weyl's theorem \citep{horn2012matrix}[Corollary 4.3.15] we can use this decomposition to bound the eigenvalues of $\krmat$ as:
    \begin{align}\label{eq:Weyl}
        \mu_i\left(\krmat_{\leqk}\right) + \mu_n\left(\krmat_{\geqk}\right) \leq \mu_i(\krmat) \leq \mu_i\left(\krmat_{\leqk}\right) + \mu_1\left(\krmat_{\geqk}\right).
    \end{align}

    Further, since $\krmat_{\leqk}=\psi_{\leqk}(X)\Sigma_{\leqk}\psi_{\leqk}(X)^\top$, we use an extension of Ostrowski's theorem, \lemref{lem:ostrowski}, to obtain the bound:
    \begin{align}\label{eq:Ostrowski}
        \lambda_{i+k-\min(n,k)}\mu_{\min(n,k)}(D_{k}) \leq \mu_i\left(\frac{1}{n}\krmat_{\leqk}\right) \leq \lambda_i \mu_1(D_k).
    \end{align}

    \noindent
    So combining the two results yields the bounds:
    \[
        \lambda_{i+k-\min(n,k)}\mu_{\min(n,k)}(D_{k}) + \mu_n\left(\frac{1}{n}\krmat_{\geqk}\right) \leq \mu_i\left(\frac{1}{n}\krmat\right) \leq \lambda_i \mu_1(D_k) + \mu_1\left(\frac{1}{n}\phi_{\geqk}(X)\phi_{\geqk}(X)^\top\right).
    \]

    \noindent
    The "in particular part" now follows from $\mu_n\left(\frac{1}{n}\krmat_{\geqk}\right) \geq 0$.
\end{proof}

\begin{lemma} \label{lem:Er_bound} 
    Suppose \assref{assumption:good_beta} holds, and that the eigenvalues of $\Sigma$ are given in non-increasing order $\lambda_1\geq \lambda_2 \geq \ldots$. Let $k\in\N$ and let $r_k$ be as defined in \defref{def:effective_rank}. There exist absolute constant $c,c'>0$ s.t it holds w.p at least $1-4 \frac{r_k}{k^4}\exp\left(-\frac{c'}{\beta_k}\frac{n}{r_k}\right)$ that
    \[
    \mu_1\left(\frac{1}{n}\phi_{\geqk}(X)\phi_{\geqk}(X)^\top\right) \leq c\left(\lambda_{k+1} + \beta_k\log(k+1)\frac{\tr\left(\Sigma_{\geqk}\right)}{n} \right).
    \]
\end{lemma}
\begin{proof}
    Let $E_k=\mu_1\left(\frac{1}{n}\phi_{\geqk}(X)\phi_{\geqk}(X)^\top\right)$, $\hat{\Sigma}_{\geqk} := \frac{1}{n}\phi_{\geqk}(X)^\top\phi_{\geqk}(X)$ and observe that $E_k=\norm{\hat{\Sigma}_{\geqk}}$. We would ideally like to bound $\norm{\hat{\Sigma}_{\geqk}}$ using the matrix Chernoff inequality with intrinsic dimension \citep{tropp2015introduction}[Theorem 7.2.1]. However, as this inequality was proved for finite matrices, if the dimension of the features is $p=\infty$ we first approximate $\norm{\hat{\Sigma}_{\geqk}}$, letting $\phi_{k+1:p'}(X):=(\phi_{k+1}(X), \ldots, \phi_p'(X))$ for some $p'\in \N$ and $\hat{\Sigma}_{k+1:p'}:=\frac{1}{n}\phi_{k+1:p'}(X)^\top\phi_{k+1:p'}(X)$, then $E_k$ can be bounded as:
    \begin{align}\label{eq:approx_sigma}
    E_k = & \norm{\frac{1}{n}\krmat_{k+1:p'} + \frac{1}{n}\krmat_{\geq p'}} \leq \norm{\frac{1}{n}\krmat_{k+1:p'}} + \norm{\frac{1}{n}\krmat_{\geq p'}} = \norm{\hat{\Sigma}_{k+1:p'}} + E_{p'}.
    \end{align}
    Furthermore, $E_{p'}$ can be bounded as
    \begin{align}\label{eq:uniform_approx}
    E_{p'} \leq & \frac{1}{n}\tr\left(\krmat_{> p'}\right)
    = \frac{1}{n}\sum_{j=1}^n\sum_{i>p'} \lambda_i \psi_i(\x_j)^2 \leq \beta_{p'}\sum_{i>p'+1} \lambda_i = \beta_{p'}\tr(\Sigma_{> p'}).
    \end{align}

    So, to summarize, either $p$ is finite, in which case we can take $p'=p$ and $E_{p'}=0$, or $p$ is infinite, in which case $E_{p'}\leq\beta_{p'}\tr(\Sigma_{> p'})$. However, by \assref{assumption:good_beta} this implies:
    \begin{align}\label{eq:small_ep}
        \forall u > 0, \exists p' \in \N \st E_{p'} \leq u.
    \end{align}

    Let $Z^{p'}_j=\frac{1}{n}\phi_{k+1:p'}(\x_j)\phi_{k+1:p'}(\x_j)^\top$ (where $(Z^{p'}_j) \succeq 0$) so that we can decompose the empirical covariance as a sum $\hat{\Sigma}_{k+1:{p'}}=\sum_{j=1}^n Z_j^{p'}$. We will need a bound on both $\mu_1(Z^{p'}_j)$ and $\mu_1(\E \hat{\Sigma}_{k+1:p'})$. 
    For the first, we have
    \[
    \mu_1(Z^{p'}_j)=\frac{1}{n}\sum_{i=k+1}^{p'} \lambda_i\psi_i(\x_j)^2 \leq \frac{1}{n}\sum_{i=k+1}^{\infty} \lambda_i\psi_i(\x_j)^2 \leq \underset{:=L}{\underbrace{\frac{\beta_k}{n}\tr(\Sigma_{\geqk})}},
    \] 
    where we denote by $L$ the right-hand side.
    For the bound on $\mu_1(\E \hat{\Sigma}_{k+1:p'})$, it holds that $\E \hat{\Sigma}_{k+1:p'}=\Sigma_{k+1:p'}=\diag(\lambda_{k+1}+1,\ldots,\lambda_p')$ and thus $\mu_1(\E \hat{\Sigma}_{k+1:p'})=\lambda_{k+1}$.

    We have shown that the conditions of \citet{tropp2015introduction}[Theorem 7.2.1] are satisfied. As such, for $r_{k:p'}:=\frac{\tr(\Sigma_{k+1:{p'}})}{\lambda_{k+1}}$ and any $t\geq 1 + \nicefrac{L}{\lambda_{k+1}}=1 + \frac{\beta_kr_k}{n}$,
    \[
    \mathbb{P}\left(\norm{\hat{\Sigma}_{k+1:{p'}}} \geq t\lambda_{k+1} \right) \leq 2r_{k:p'} \left(\frac{e^{t-1}}{t^{t}}\right)^{\nicefrac{\lambda_{k+1}}{L}}.
    \]
    By \eqref{eq:approx_sigma} it holds that $\norm{\hat{\Sigma}_{k+1:{p'}}} \geq E_k - E_{p'}$. Using this, the fact that $\frac{\lambda_{k+1}}{L}=\frac{n}{\beta_kr_k}$, and upper bounding $e^{t-1}\leq e^t$, $r_{k: p'}\leq r_k$ yields 
    \[
    \mathbb{P}\left(E_k - E_{p'} \geq t\lambda_{k+1} \right) \leq \mathbb{P}\left(\norm{\hat{\Sigma}_{k+1:{p'}}} \geq t\lambda_{k+1} \right) \leq 2r_k \left(\frac{e}{t}\right)^{\nicefrac{tn}{\beta_kr_k}}.
    \]

    \noindent
    Now we pick $t=e^3 + 2\frac{\beta_kr_k}{n}\log(k+1)$, (which satisfies the requirement of $t\geq 1 + \frac{\beta_kr_k}{n}$). In particular $\frac{e}{t}\leq \frac{1}{e^2}$, and we obtain that: 
    \begin{align*}
    \mathbb{P}\left(E_k \geq t\lambda_{k+1} + E_{p'} \right) \leq & 2r_k \left(\frac{1}{e^2}\right)^{\frac{e^3}{\beta_k}\frac{n}{r_k} + 2\log(k+1)} \\ 
    \leq & 2 \frac{r_k}{(k+1)^4}\exp\left(-2\frac{e^3}{\beta_k}\frac{n}{r_k}\right).
    \end{align*}

    \noindent
    Furthermore, $E_{p'}$ can be bounded via \eqref{eq:uniform_approx}As a result, we obtain that for $c'=2e^3$, $c=e^3$, it holds w.p at least $1-4 \frac{r_k}{k^4}\exp\left(-\frac{c'}{\beta_k}\frac{n}{r_k}\right)$ that
    \[
    E_k \leq c\left(\lambda_{k+1} + \beta_k\log(k+1)\frac{\tr\left(\Sigma_{\geqk}\right)}{n}+ E_{p'}\right).
    \]
    Notice that the bound on $E_k$ depends on $p'$ only via $E_{p'}$. So by \eqref{eq:small_ep} we are done.

\end{proof}

\begin{lemma}\label{lem:eff_regularizaion}
    Let $R_k$ be as defined in \defref{def:effective_rank}. For any $\delta > 0$ it holds w.p at least $1- \delta$ that for all $1 \leq i\leq n$
    \[
    \alpha_k\frac{1}{n}\tr\left(\Sigma_{\geqk}\right)\left(1-\frac{1}{\delta}\sqrt{\frac{n^2}{R_k}}\right) \leq \mu_i\left(\frac{1}{n}\krmat_{\geqk}\right) \leq \beta_k\frac{1}{n}\tr\left(\Sigma_{\geqk}\right)\left(1+\frac{1}{\delta}\sqrt{\frac{n^2}{R_k}}\right).
    \]
\end{lemma}
\begin{proof}
   Let $\Lambda_{\geqk}:=\diag(\frac{1}{n}\krmat_{\geqk})\in\R^{n\times n}$ be equal to $\frac{1}{n}\krmat_{\geqk}$ on the diagonal and $0$ elsewhere, and $\Delta_{\geqk}:=\frac{1}{n}\krmat_{\geqk} - \Lambda_{\geqk}$ be the remainder. $\Lambda_{\geqk}$ is a diagonal matrix with the $i$'th value on the diagonal given by $[\Lambda_{\geqk}]_{ii} = \frac{1}{n}\sum_{\ell\geqk} \lambda_\ell \psi_\ell(\x_i)^2$. By \defref{def:eigen_lower} of $\alpha_k$ and \defref{def:eigen_upper} of $\beta_k$ it holds that 
   \[
   \alpha_k\frac{1}{n}\tr\left(\Sigma_{\geqk}\right) \leq [\Lambda_{\geqk}]_{ii} \leq \beta_k\frac{1}{n}\tr\left(\Sigma_{\geqk}\right),
   \]
   which together with the fact that $\Lambda_{\geqk}$ is diagonal implies
   \begin{align}
       \alpha_k\frac{1}{n}\tr\left(\Sigma_{\geqk}\right)I \preceq \Lambda_{\geqk} \preceq \beta_k\frac{1}{n}\tr\left(\Sigma_{\geqk}\right)I.
   \end{align}
   As such, by Weyl's theorem \citep{horn2012matrix}[Corollary 4.3.15], we can bound the eigenvalues of $\frac{1}{n}\krmat_{\geqk}$ as
    \begin{align}\label{eq:helper_e_k}
        \alpha_k\frac{1}{n}\tr\left(\Sigma_{\geqk}\right) + \mu_{n}\left(\Delta_{\geqk}\right) \leq \mu_i\left(\frac{1}{n}\krmat_{\geqk}\right)\leq \beta_k\frac{1}{n}\tr\left(\Sigma_{\geqk}\right) + \mu_{1}\left(\Delta_{\geqk}\right).
    \end{align}
    So in order to bound the eigenvalues of $\frac{1}{n}K_{\geqk}$, it remains to bound the eigenvalues of $\Delta_{\geqk}$. We first bound the expectation using
    \begin{align*}
        \E[\norm{\Delta_{\geqk}}] \leq & \E[\norm{\Delta_{\geqk}}_F^2]^\frac{1}{2} = \sqrt{\sum_{i,j=1}^n \E \left[\left(\frac{1}{n}\langle \phi_{\geqk}(\x_i), \phi_{\geqk}(\x_j) \rangle\right)^2\right]} \\ 
        =& \sqrt{\frac{n(n-1)}{n^2}\tr\left(\Sigma_{\geqk}^2\right)} \leq \sqrt{\tr\left(\Sigma_{\geqk}^2\right)} = \frac{1}{n}\tr\left(\Sigma_{\geqk}\right)\sqrt{\frac{n^2}{R_k}}.
    \end{align*}
    By Markov's inequality, it holds that 
    \[
    \mathbb{P}\left(\norm{\Delta_{\geqk}} \geq \frac{1}{\delta}\E[\norm{\Delta_{\geqk}}]\right) \leq \delta.
    \]
    Implying that with probability at least $1-\delta$ it holds that
    \[
    \norm{\Delta_{\geqk}} \leq \frac{1}{\delta} \E[\norm{\Delta_{\geqk}}] \leq \frac{1}{n\delta}\tr\left(\Sigma_{\geqk}\right)\sqrt{\frac{n^2}{R_k}}.
    \]
    Finally, plugging this back into \eqref{eq:helper_e_k} completes the proof.

\end{proof}
\begin{corollary}\label{cor:ak_eigen_bound}
    Suppose \assref{assumption:good_beta} holds, and that the eigenvalues of $\Sigma$ are given in non-increasing order $\lambda_1\geq \lambda_2 \geq \ldots$. Let $k\in\N$ and let $r_k$ be as defined in \defref{def:effective_rank}. There exist absolute constant $c,c'>0$ s.t it holds w.p at least $1-4 \frac{r_k}{k^4}\exp\left(-\frac{c'}{\beta_k}\frac{n}{r_k}\right)$ that
    \[
    \mu_1\left(\frac{1}{n}\krmat_{\geqk}\right) \leq c\left(\lambda_{k+1} + \beta_k\log(k+1)\frac{\tr\left(\Sigma_{\geqk}\right)}{n}\right).
    \]
    And for any $k'\in\N$ with $k'>k$, and any $\delta > 0$ it holds w.p at least $1- \delta$ that
    \[
    \alpha_k\left(1-\frac{1}{\delta}\sqrt{\frac{n^2}{R_{k'}}}\right) 
    \frac{\tr\left(\Sigma_{\geqk'}\right)}{n}\leq \mu_n\left(\frac{1}{n}\krmat_{\geqk'}\right), 
    \]
    so that both statements hold w.p at least  $1-\delta - 4 \frac{r_k}{k^4}\exp\left(-\frac{c'}{\beta_k}\frac{n}{r_k}\right)$.
\end{corollary}
\begin{proof}
     By Weyl's theorem \citep{horn2012matrix}[Corollary 4.3.15], for any $k'\geq k, \mu_n(\krmat_{\geq k})\geq \mu_n(\krmat_{\geq k'})+\mu_n(\krmat_{k:k'}) \geq \mu_n(\krmat_{\geq k'})$. So the lower bound comes from \lemref{lem:eff_regularizaion} (with $k'$) and the upper bound comes from \lemref{lem:Er_bound}.
\end{proof}

\section{Upper bounds for the Risk - Proofs of Results in \secref{sec:krr}}\label{app:krr}
\subsection{Proof of \thmref{thm:bound_gen}.}
The majority of the work was done in lemmas \ref{lem:concentration_union}, \ref{lem:concentration3}, \ref{lem:variance} and \ref{lem:bias}. Here we essentially combine these results to obtain the desired bounds. Throughout the section, the notations of $A_k:=\krmat_{\geqk}+n\reg I$ and $A:=\krmat+n\reg I$ as defined in \secref{app:notations} will be very common.
\boundgen*

\begin{proof}
    The majority of the work is given by lemmas \ref{lem:bound_var} and \ref{lem:bound_bias}. We note a few properties which are immediate, from which the claim will follow:
        \begin{align}\label{eq:bound_helper1}
            \frac{\mu_1\left(\frac{1}{n}A_k\right)^2}{\mu_n\left(\frac{1}{n}A_k\right)^2} 
            =& \left(\frac{\mu_1\left(\frac{1}{n}\krmat_{\geqk}\right) + \reg}{\mu_n\left(\frac{1}{n}\krmat_{\geqk}\right) + \reg}\right)^2 \leq \err_{k,n}^2.
        \end{align}

        \begin{align}\label{eq:bound_helper2}
            \frac{\norm{\Sigma_{\geqk}}}{\mu_n\left(\frac{1}{n}A_k\right)} \leq \err_{k,n}.
        \end{align}

        \begin{align}\label{eq:bound_helper3}
            \frac{1}{n\mu_n\left(\frac{1}{n}A_k\right)^2}\sum_{i\geqk}\lambda_i^2 
            = & \frac{\norm{\Sigma_{\geqk}}^2}{\mu_n\left(\frac{1}{n}A_k\right)^2} \cdot \frac{r_k(\Sigma^2)}{n}
            \leq \err_{k,n}^2 \frac{r_k(\Sigma^2)}{n}.
        \end{align}

    Furthermore, because the trace of a matrix is the sum of its eigenvalues, we obtain
    \begin{align}\label{eq:bound_helper4}
        \mu_1\left(\frac{1}{n}A_k\right)^2 = & \frac{\mu_1\left(\frac{1}{n}A_k\right)^2}{\mu_n\left(\frac{1}{n}A_k\right)^2} \mu_n\left(\frac{1}{n}A_k\right)^2 \leq \err_{k,n}^2 \left(\frac{1}{n}\tr\left(\frac{1}{n}A_k\right)\right)^2 \nonumber \\
        \leq & \err_{k,n}^2 \left(\reg + \frac{1}{n^2}\sum_{j=1}^n\sum_{i\geqk}\lambda_i\psi_i(\x_j)^2\right)^2 \leq \err_{k,n}^2 \left(\reg + \frac{\beta_k\tr\left(\Sigma_{\geqk}\right)}{n}\right)^2.
    \end{align} 
     and similarly 
    \begin{align}\label{eq:bound_helper5}
        \mu_n\left(\frac{1}{n}A_k\right)^2 \geq & \frac{\mu_n\left(\frac{1}{n}A_k\right)^2}{\mu_1\left(\frac{1}{n}A_k\right)^2} \mu_1\left(\frac{1}{n}A_k\right)^2 \geq \frac{1}{\err_{k,n}^2} \left(\frac{1}{n}\tr\left(\frac{1}{n}A_k\right)\right)^2 \nonumber \\
        \geq & \frac{1}{\err_{k,n}^2} \left(\reg + \frac{1}{n^2}\sum_{j=1}^n\sum_{i\geqk}\lambda_i\psi_i(\x_j)^2\right)^2 \geq \frac{1}{\err_{k,n}^2} \left(\reg + \frac{\alpha_k\tr\left(\Sigma_{\geqk}\right)}{n}\right)^2.
    \end{align} 
    We thus also obtain an alternative bound for \eqref{eq:bound_helper3} via  \eqref{eq:bound_helper5} as
    \begin{align}\label{eq:bound_helper3_v2}
            \frac{1}{n\mu_n\left(\frac{1}{n}A_k\right)^2}\sum_{i\geqk}\lambda_i^2 
            \leq & \err_{k,n}^2 \frac{n\sum_{i\geqk}\lambda_i^2 }{\left(n\reg + \alpha_k\tr\left(\Sigma_{\geqk}\right)\right)^2} \leq \frac{\err_{k,n}^2}{\alpha_k^2}\frac{n}{R_k(\Sigma)}.
        \end{align}
    
    Now for the variance part of the claim, by combining \lemref{lem:bound_var}  with \eqref{eq:bound_helper1}, \eqref{eq:bound_helper3} and \eqref{eq:bound_helper3_v2}, we obtain that w.p at least $1-\delta - 8\exp\left(-\frac{c'}{\beta_k^2}\frac{n}{k}\right)$ it holds that
    \begin{align}
    V \leq & C_1\err_{k,n}^2\sigma_\epsilon^2 \left(\frac{k}{n} +  \min\left(\frac{r_k\left(\Sigma^2\right)}{n}, \frac{n}{\alpha_k^2R_k(\Sigma)}\right)\right).
    \end{align}

    For the bias part of the claim, by similarly combining \lemref{lem:bound_bias}  with \eqref{eq:bound_helper1}, \eqref{eq:bound_helper2} and \eqref{eq:bound_helper4}, and using the fact that $\err_{k,n}>1$, we obtain that w.p at least $1-\delta - 8\exp\left(-\frac{c'}{\beta_k^2}\frac{n}{k}\right)$
    \begin{align}
        B\leq & C_2 \Biggl(\norm{\theta^*_\geqk}_{\Sigma_\geqk}^2 \left(1 + \frac{1}{\delta}\left(\err_{k,n}^2 + \err_{k,n}\right)\right)\Biggr. \nonumber\\ 
        & \Biggl. \qquad  + \norm{\theta_\leqk^*}_{\Sigma_\leqk^{-1}}^2\left(\err_{k,n}^2\left(\reg + \frac{ \beta_k\tr\left(\Sigma_{\geqk}\right)}{n}\right)^2\left(1+\err_{k,n}\right)\right)\Biggr) \nonumber \\
        \leq & C_2\cdot 3\err_{k,n}^3\left(\frac{1}{\delta}\norm{\theta^*_\geqk}_{\Sigma_\geqk}^2 + \norm{\theta_\leqk^*}_{\Sigma_\leqk^{-1}}^2 \left(\reg + \frac{ \beta_k\tr\left(\Sigma_{\geqk}\right)}{n}\right)^2\right).
    \end{align}
    So everything holds w.p at least $1-\delta - 16\exp\left(-\frac{c'}{\beta_k^2}\frac{n}{k}\right)$
\end{proof}

\begin{lemma}\label{lem:bound_var}
    There exists some absolute constants $c, c', C_1>0$, s.t for any $k\in\N$ with $c\beta_kk\log(k)\leq n$, it holds w.p at least $1-8\exp\left(-\frac{c'}{\beta_k^2}\frac{n}{k}\right)$ the variance can be upper bounded as:
\begin{align}
    V \leq & C_1\sigma_\epsilon^2 \left(\frac{\mu_1\left(\frac{1}{n}A_k\right)^2 k}{\mu_n\left(\frac{1}{n}A_k\right)^2 n} +  \frac{1}{n\mu_n\left(\frac{1}{n}A_k\right)^2}\sum_{i\geqk}\lambda_i^2\right).
\end{align}
\end{lemma}

\begin{proof}
    $A_k$ is positive definite for any $\reg>0$ and thus, by lemma $\eqref{lem:variance}$ we have that:
    \begin{align*}
        V \leq \sigma_\epsilon^2\left(\frac{\mu_1(A_k^{-1})^2\tr( \psi_\leqk(X)\psi_\leqk(X)^\top)}{\mu_n(A_k^{-1})^2 \mu_k\left(\psi_\leqk(X)^\top \psi_\leqk(X)\right)^2} +  \mu_1(A_k^{-1})^2\tr(\phi_\geqk(X)\Sigma_\geqk \phi_\geqk(X)^\top)\right).
    \end{align*}

    Plugging in the bounds from \lemref{lem:concentration_union}, there are some absolute constants $c, c', c_1, c_2>0$ s.t for any $k\in\N$ with $c\beta_kk\log(k)\leq n$, it holds w.p at least $1-8\exp\left(-\frac{c'}{\beta_k^2}\frac{n}{k}\right)$ that

    \begin{align*}
        V \leq & \sigma_{\epsilon}^2\left(\frac{\mu_1(A_k^{-1})^2c_2kn}{\mu_n(A_k^{-1})^2 c_1^2n^2} +  \mu_1(A_k^{-1})^2c_2 n \sum_{i\geqk}\lambda_i^2\right) \\
        \leq & c_2\left(\frac{1}{c_1^2}+1\right)\sigma_{\epsilon}^2 \left(\frac{\mu_1(A_k^{-1})^2 k}{\mu_n(A_k^{-1})^2 n} +  \mu_1(A_k^{-1})^2 n \sum_{i\geqk}\lambda_i^2\right).
    \end{align*}

    Now taking $C_1$ accordingly, and the facts that $\mu_1(A_k^{-1})=\frac{1}{n\mu_n(\frac{1}{n}A_k)}$ and $\mu_n(A_k^{-1})=\frac{1}{n\mu_1(\frac{1}{n}A_k)}$ complete the proof.

\end{proof}

\begin{lemma}\label{lem:bound_bias}
    There exists some absolute constants $c, c', C_2>0$ (where $c$ and $c'$ are the same as in \lemref{lem:bound_var}),  s.t for any $k\in\N$ with $c\beta_kk\log(k)\leq n$, and $\delta>0$, it holds w.p at least $1-\delta - 8\exp\left(-\frac{c'}{\beta_k^2}\frac{n}{k}\right)$ the bias can be upper bounded as:
    \begin{align}
        B\leq & C_2 \Biggl(\norm{\theta^*_\geqk}_{\Sigma_\geqk}^2 \left(1 + \frac{1}{\delta}\left(\frac{\mu_1(A_k^{-1})^2}{\mu_n(A_k^{-1})^2} + \frac{\norm{\Sigma_{\geqk}}}{\mu_n\left(\frac{1}{n}A_k\right)}\right)\right)\Biggr. \nonumber\\ 
        & \Biggl. \qquad  + \norm{\theta_\leqk^*}_{\Sigma_\leqk^{-1}}^2\left(\mu_1\left(\frac{1}{n}A_k\right)^2\left(1+\frac{\norm{\Sigma_{\geqk}}}{\mu_n\left(\frac{1}{n}A_k\right)}\right)\right)\Biggr).
    \end{align}
\end{lemma}

\begin{proof}
    Similarly, to the variance term, by lemma \eqref{lem:bias} we have that
    \begin{align*}
    &\lefteqn{\|\theta^* - \hat{\theta}( \phi(X)\theta^*)\|_\Sigma^2} \\
    & \leq \|\theta^*_\geqk\|_{\Sigma_\geqk}^2+ \frac{\mu_1(A_k^{-1})^2}{\mu_n(A_k^{-1})^2}\frac{\mu_1\left(\psi_{\leqk}(X)^{\top} \psi_{\leqk}(X) \right)}{\mu_k\left(\psi_{\leqk}(X)^{\top} \psi_{\leqk}(X) \right)^2}\|\phi_\geqk(X) \theta^*_\geqk\|^2 \\ &\qquad +\frac{\|\theta_\leqk^*\|_{\Sigma_\leqk^{-1}}^2}{\mu_n(A_k^{-1})^2\mu_k\left(\psi_{\leqk}(X)^{\top} \psi_{\leqk}(X) \right)^2}\\
    &\qquad + \norm{\Sigma_{\geqk}} \mu_1(A_k^{-1})\|\phi_\geqk(X)\theta^*_\geqk\|^2\\
    &\qquad + \norm{\Sigma_{\geqk}} \frac{\mu_1(A_k^{-1})}{\mu_n(A_k^{-1})^2} \frac{\mu_1(\psi_{\leqk}(X)^{\top} \psi_{\leqk}(X))}{\mu_k(\psi_{\leqk}(X)^{\top} \psi_{\leqk}(X))^2}\|\Sigma_{\leqk}^{-1/2}\theta^*_\leqk\|^2.
    \end{align*}

    Plugging in the bounds from lemmas \eqref{lem:concentration_union} and \eqref{lem:concentration3}, there are some absolute constants $c, c', c_1, c_2>0$ s.t for any $k\in\N$ with $c\beta_kk\log(k)\leq n$, it holds w.p at least $1-8\exp\left(-\frac{c'}{\beta_k^2}\frac{n}{k}\right)$ that
    
    \begin{align*}
    &\lefteqn{\|\theta^* - \hat{\theta}( \phi(X)\theta^*)\|_\Sigma^2} \\
    & \leq \norm{\theta^*_\geqk}_{\Sigma_\geqk}^2
    + \frac{\mu_1(A_k^{-1})^2}{\mu_n(A_k^{-1})^2} \frac{c_2n}{c_1^2n^2} \cdot \frac{1}{\delta}n\norm{\theta_\geqk^*}_{\Sigma_\geqk}^2 \\ 
    &\qquad +\frac{\|\theta_\leqk^*\|_{\Sigma_\leqk^{-1}}^2}{\mu_n(A_k^{-1})^2c_1^2n^2}\\
    &\qquad + \norm{\Sigma_{\geqk}} \mu_1(A_k^{-1})\cdot \frac{1}{\delta}n\norm{\theta_\geqk^*}_{\Sigma_\geqk}^2\\
    &\qquad + \norm{\Sigma_{\geqk}} \frac{\mu_1(A_k^{-1})}{\mu_n(A_k^{-1})^2} \frac{c_2n}{c_1^2n^2}\|\Sigma_{\leqk}^{-1/2}\theta^*_\leqk\|^2 \\
    &\leq C_2 \left(\norm{\theta^*_\geqk}_{\Sigma_\geqk}^2 \left(1 + \frac{1}{\delta}\left(\frac{\mu_1(A_k^{-1})^2}{\mu_n(A_k^{-1})^2} + n\norm{\Sigma_{\geqk}} \mu_1(A^{-1})\right)\right)\right. \\ 
    &\left. \qquad\quad  + \norm{\theta_\leqk^*}_{\Sigma_\leqk^{-1}}^2\left(\frac{1}{n^2\mu_n(A_k^{-1})^2} + \norm{\Sigma_{\geqk}}\frac{\mu_1(A_k^{-1})}{n\mu_n(A_k^{-1})^2} \right)\right),
    \end{align*}
    where $C_2>0$ can be chosen to depend only on $c_1$ and $c_2$ (which are absolute constants). Now we can use the facts that $\mu_1(A_k^{-1})=\frac{1}{n\mu_n(\frac{1}{n}A_k)}$ and $\mu_n(A_k^{-1})=\frac{1}{n\mu_1(\frac{1}{n}A_k)}$ to complete the proof, since $\mu_1(A^{-1})\leq \mu_1(A_k^{-1})=\frac{1}{n\mu_n(\frac{1}{n}A_k)}$ and $\frac{1}{n^2\mu_n(A_k^{-1})^2}=\mu_1(A_k^{-1})^2$, and finally $\frac{\mu_1(A_k^{-1})}{n\mu_n(A_k^{-1})^2}=\frac{1}{\mu_n(A_k^{-1})}$.
\end{proof}

\subsection{Lemmas for Risk bounds}
In \citet{tsigler2023benign}[Appendices F,G,H], several inequalities which will be highly useful to us were derived. Unfortunately, they assumed throughout their paper that the features are finite-dimensional, mean zero, and follow some sub-Gaussianity constraint. The proofs from their paper that we need technically do not depend on these constraints. However, for completeness and rigor, we rewrite their proofs here, adjusted where necessary to match our settings. Again, we remind the reader of the notations $A_k:=\krmat_{\geqk}+n\reg I$ and $A:=\krmat+n\reg I$ as defined in \secref{app:notations}.

\begin{lemma} \label{lem:leqk_pred}
    For any $k\in \N$ it holds that
    \begin{align*}
        \hat{\theta}( y)_\leqk + \phi_\leqk(X)^\top A_k^{-1} \phi_\leqk(X)\hat{\theta}( y)_\leqk = \phi_\leqk(X)^\top A_k^{-1}y.
    \end{align*}
\end{lemma}
\begin{proof}
We start with the ridgeless case, where $\hat{\theta}(y)$ is the minimum norm interpolating solution. Note that $\hat{\theta}(y)_\geqk$ is also the minimum norm solution to the equation $\phi_\geqk(X)\theta_\geqk = y - \phi_\leqk(X) \hat{\theta}(y)_\leqk$, where $\theta_\geqk$ is the variable. Thus, we can write
\[
\hat{\theta}(y)_\geqk = \phi_\geqk(X)^\top \left(\phi_\geqk(X) \phi_\geqk(X)^\top\right)^{-1}\left(y - \phi_\leqk(X) \hat{\theta}(y)_\leqk\right).
\]

As such, we obtain that the min norm interpolator is the minimizer of the following:
\[
\hat{\theta}(y) = \argmin_{\theta_{\leqk}} v(\theta_\leqk) := \Bigl[\theta_\leqk^\top, \left(y - \phi_\leqk(X) \theta_\leqk\right)^\top\left(\phi_\geqk(X) \phi_\geqk(X)^\top\right)^{-1}\phi_\geqk(X)\Bigr]
\]

As $\theta_\leqk$ varies, this vector sweeps an affine subspace of our Hilbert space. The vector $\hat{\theta}(y)_\leqk$ gives the minimum norm if and only if for any additional vector $\eta_\leqk$ we have $v(\hat{\theta}(y)_\leqk) \perp v(\hat{\theta}(y)_\leqk + \eta_\leqk) - v(\hat{\theta}(y)_\leqk)$. Let's write out the second vector: $\forall \eta_\leqk \in \R^k$
\[
 v(\hat{\theta}(y)_\leqk + \eta_\leqk) - v(\hat{\theta}(y)_\leqk) = 
\Bigl[\eta_\leqk^\top,  - \eta_\leqk^\top \phi_\leqk(X)^\top\left(\phi_\geqk(X) \phi_\geqk(X)^\top\right)^{-1}\phi_\geqk(X)\Bigr]
\]
We see that the above mentioned orthogonality for any $\eta_\leqk$ is equivalent to the following:
\begin{align*}
\hat{\theta}(y)_\leqk^\top - \left(y - \phi_\leqk(X) \hat{\theta}(y)_\leqk\right)^\top\left(\phi_\geqk(X) \phi_\geqk(X)^\top\right)^{-1}\phi_\leqk(X) &= 0, \\
\hat{\theta}(y)_\leqk + \phi_\leqk(X)^\top A_k^{-1} \phi_\leqk(X)\hat{\theta}(y)_\leqk &= \phi_\leqk(X)^\top A_k^{-1}y, \\
\end{align*}
where we replaced $\phi_\geqk(X) \phi_\geqk(X)^\top = :A_k$.

This completes the ridgeless case, and we now move on to the case of $\reg > 0$. We have that
\begin{align*}
\hat{\theta}(y)_\leqk = \phi_\leqk(X)^\top(\krmat+n\reg I)^{-1}y = \phi_\leqk(X)^\top(A_k + \phi_\leqk(X) \phi_\leqk(X)^\top)^{-1}y.
\end{align*}
Which yields
\begin{align*}
&\hat{\theta}(y)_\leqk + \phi_\leqk(X)^\top A_k^{-1} \phi_\leqk(X)\hat{\theta}(y)_\leqk\\
=& \phi_\leqk(X)^\top(A_k + \phi_\leqk(X) \phi_\leqk(X)^\top)^{-1}y \\
& + \phi_\leqk(X)^\top A_k^{-1} \phi_\leqk(X) \phi_\leqk(X)^\top(A_k + \phi_\leqk(X) \phi_\leqk(X)^\top)^{-1}y\\
=&\phi_\leqk(X)^\top A_k^{-1}( A_k + \phi_\leqk(X) \phi_\leqk(X)^\top)(A_k + \phi_\leqk(X) \phi_\leqk(X)^\top)^{-1}y\\
=&\phi_\leqk(X)^\top A_k^{-1}y.
\end{align*}
\end{proof}

We now prove a very simple lemma that will help us formalized the intuition that we can split the error into the $\leqk$ and $\geqk$ components
\begin{lemma}\label{lem:split}
    For any $k\in \N$, and $\vv\in\R^{\N}$,
    \[
    \norm{\vv}_{\Sigma}^2 = \norm{\vv_\leqk}_{\Sigma_\leqk}^2 + \norm{\vv_\geqk}_{\Sigma_\geqk}^2
    \]
\end{lemma}
\begin{proof} We can write $\vv=\begin{bmatrix} \vv_\leqk \\ \vv_\geqk \\ \end{bmatrix}$ and since $\Sigma$ is diagonal $\Sigma = \begin{bmatrix} \Sigma_\leqk & 0 \\ 0 & \Sigma_\geqk \\ \end{bmatrix}$ and thus:
\[
 \norm{\vv}_{\Sigma}^2 =
\begin{bmatrix} \vv_\leqk & \vv_\geqk \\ \end{bmatrix} \begin{bmatrix} \Sigma_\leqk & 0 \\ 0 & \Sigma_\geqk \\ \end{bmatrix} \begin{bmatrix} \vv_\leqk \\ \vv_\geqk \\ \end{bmatrix} 
=  \norm{\vv_\leqk}_{\Sigma_\leqk}^2 + \norm{\vv_\geqk}_{\Sigma_\geqk}^2.
\]
\end{proof}

The next lemma provides a useful upper bound for the variance.
\begin{lemma}[Variance term]
\label{lem:variance}
If for some $k \in \N$ the matrix $A_k$ is PD, then
\[
V \leq \sigma_\epsilon^2\left(\frac{\mu_1(A_k^{-1})^2\tr( \psi_\leqk(X) \psi_\leqk(X)^\top)}{\mu_n(A_k^{-1})^2 \mu_k\left(\psi_\leqk(X)^\top \psi_\leqk(X) \right)^2} +  \mu_1(A_k^{-1})^2\tr(\phi_\geqk(X)\Sigma_\geqk \phi_\geqk(X)^\top)\right).
\]
\end{lemma}

\begin{proof}

Recall that 
\[
V = \E_{\epsilon}\left[\norm{\hat{\theta}(\bepsilon)}^2_{\Sigma}\right] = \E_{\epsilon}\left[\norm{\phi(X)^\top (\krmat+n\reg I)^{-1}\epsilon}^2_{\Sigma} \right]
\]
By Lemma \eqref{lem:split} we can split the variance into $\norm{\hat{\theta}( \bepsilon_\leqk)}^2_{\Sigma\leqk}$ and $\norm{\hat{\theta}( \bepsilon_\geqk)}^2_{\Sigma\geqk}$ and bound these separately.

Lemma \eqref{lem:leqk_pred} states that
\[
\phi_\leqk(X)^\top A_k^{-1}\epsilon = \hat{\theta}( \bepsilon_\leqk) + \phi_\leqk(X)^\top A_k^{-1}\phi_\leqk(X) \hat{\theta}( \bepsilon_\leqk).
\]
Multiplying the identity by $\hat{\theta}( \bepsilon_\leqk)^\top$ from the left, and using that $\hat{\theta}( \bepsilon_\leqk)^\top \hat{\theta}( \bepsilon_\leqk) \geq 0$ we get
\begin{equation}\label{eqn:firstthetahatineq}
\hat{\theta}( \bepsilon_\leqk)^\top \phi_\leqk(X)^\top A_k^{-1}\epsilon 
 \geq \hat{\theta}( \bepsilon_\leqk)^\top \phi_\leqk(X)^\top A_k^{-1}\phi_\leqk(X) \hat{\theta}( \bepsilon_\leqk).
\end{equation}

The leftmost expression is linear in $\hat{\theta}( \bepsilon_\leqk)$, and the rightmost is quadratic. We use these expressions to bound  $\|\hat{\theta}( \bepsilon_\leqk)\|_{\Sigma_\leqk}$.

First, we extract that norm from the quadratic part
\begin{align*}
\hat{\theta}( \bepsilon_\leqk)^\top \phi_\leqk(X)^\top A_k^{-1}\phi_\leqk(X) \hat{\theta}( \bepsilon_\leqk)
\geq& \mu_n(A_k^{-1})\hat{\theta}( \bepsilon_\leqk)^\top \phi_\leqk(X)^\top \phi_\leqk(X) \hat{\theta}( \bepsilon_\leqk) \\
\geq& \mu_n(A_k^{-1}) \|\hat{\theta}(\bepsilon_\leqk)\|_{\Sigma_\leqk}^2\mu_k\left(\psi_\leqk(X)^\top \psi_\leqk(X) \right).
\end{align*}

Then we can substitute~\eqref{eqn:firstthetahatineq} and apply Cauchy-Schwarz to obtain
\begin{align*}
\|\hat{\theta}( \bepsilon_\leqk)\|_{\Sigma_\leqk}^2\mu_n(A_k^{-1})\mu_k\left(\psi_\leqk(X)^\top \psi_\leqk(X) \right)
& \leq \hat{\theta}( \bepsilon_\leqk)^\top \phi_\leqk(X)^\top A_k^{-1}\phi_\leqk(X) \hat{\theta}( \bepsilon_\leqk) \\
& \leq \hat{\theta}( \bepsilon_\leqk)^\top \phi_\leqk(X)^\top A_k^{-1}\epsilon \\
& \leq  \|\hat{\theta}( \bepsilon_\leqk)\|_{\Sigma_\leqk}\left\|\psi_\leqk(X)^\top A_k^{-1}\epsilon\right\|,
\end{align*}
and so
\[
\|\hat{\theta}( \bepsilon_\leqk)\|_{\Sigma_\leqk}^2 \leq \frac{\epsilon^\top A_k^{-1}\psi_\leqk(X)\psi_\leqk(X)^\top A_k^{-1}\epsilon}{\mu_n(A_k^{-1})^2 \mu_k\left(\psi_\leqk(X)^\top \psi_\leqk(X) \right)^2}.
\]
Since $\epsilon$ is independent of $X$, taking expectation in $\epsilon$ only leaves the trace in the numerator:
\begin{align*}
\E_\epsilon\|\hat{\theta}( \bepsilon_\leqk)\|_{\Sigma_\leqk}^2 
\leq& \sigma_\epsilon^2\frac{\tr( A_k^{-1}\psi_\leqk(X)\psi_\leqk(X)^\top A_k^{-1})}{\mu_n(A_k^{-1})^2 \mu_k\left(\psi_\leqk(X)^\top \psi_\leqk(X) \right)^2}\\
\leq& \sigma_\epsilon^2 \frac{\mu_1(A_k^{-1})^2\tr( \psi_\leqk(X)\psi_\leqk(X)^\top)}{\mu_n(A_k^{-1})^2 \mu_k\left(\psi_\leqk(X)^\top \psi_\leqk(X) \right)^2},
\end{align*}
where we transitioned to  the second line by using the fact that $\tr(MM'M)\leq \mu_1(M)^2\tr(M')$ for PD matrices $M, M'$.

This completes the bound for the first $\leqk$ components, and we now move on to the $\geqk$ ones. The rest of the variance term is 
\[
\left\|\Sigma_\geqk^{1/2}\phi_\geqk(X)^\top A^{-1}\epsilon\right\|^2 = \epsilon^\top A^{-1}\phi_\geqk(X)\Sigma_\geqk \phi_\geqk(X)^\top A^{-1}\epsilon.
\]
Since $\epsilon$ is independent of $X$, taking expectation in $\epsilon$ only leaves the trace of the matrix:
\begin{align*}
\frac{1}{\sigma_\epsilon^2}\E_\epsilon\left\|\Sigma_\geqk^{1/2}\phi_\geqk(X)^\top A^{-1}\epsilon\right\|^2 =& \tr( A^{-1}\phi_\geqk(X)\Sigma_\geqk \phi_\geqk(X)^\top A^{-1})\\
\leq& \mu_1(A^{-1})^2\tr(\phi_\geqk(X)\Sigma_\geqk \phi_\geqk(X)^\top)\\
\leq& \mu_1(A_k^{-1})^2\tr(\phi_\geqk(X)\Sigma_\geqk \phi_\geqk(X)^\top).\\
\end{align*}
Here we again used the fact that $\tr(MM'M)\leq \mu_1(M)^2\tr(M')$ for PD matrices $M, M'$ to transition to the second line. We then used $A \succeq A_k$ to infer $\mu_1(A^{-1}) \leq \mu_1(A_k^{-1})$. 
\end{proof}

We now move on to bounding the bias term.
\begin{lemma}[Bias term]
\label{lem:bias}
Suppose that for some $k < n$ the matrix $A_k$ is PD. Then, 
\begin{align*}
&\lefteqn{\|\theta^* - \hat{\theta}( \phi(X)\theta^*)\|_\Sigma^2} \\
& \leq \|\theta^*_\geqk\|_{\Sigma_\geqk}^2+ \frac{\mu_1(A_k^{-1})^2}{\mu_n(A_k^{-1})^2}\frac{\mu_1\left(\psi_{\leqk}(X)^{\top} \psi_{\leqk}(X) \right)}{\mu_k\left(\psi_{\leqk}(X)^{\top} \psi_{\leqk}(X) \right)^2}\|\phi_\geqk(X) \theta^*_\geqk\|^2 \\ &\qquad +\frac{\|\theta_\leqk^*\|_{\Sigma_\leqk^{-1}}^2}{\mu_n(A_k^{-1})^2\mu_k\left(\psi_{\leqk}(X)^{\top} \psi_{\leqk}(X) \right)^2}\\
&\qquad + \norm{\Sigma_{\geqk}} \mu_1(A_k^{-1})\|\phi_\geqk(X)\theta^*_\geqk\|^2\\
&\qquad + \norm{\Sigma_{\geqk}} \frac{\mu_1(A_k^{-1})}{\mu_n(A_k^{-1})^2} \frac{\mu_1(\psi_{\leqk}(X)^{\top} \psi_{\leqk}(X))}{\mu_k(\psi_{\leqk}(X)^{\top} \psi_{\leqk}(X))^2}\|\Sigma_{\leqk}^{-1/2}\theta^*_\leqk\|^2.
\end{align*}

\end{lemma}

\begin{proof}
As before, by Lemma \eqref{lem:split} we can bound the $\leqk$ components and the $\geqk$ components separately. 
We start by bounding $\|\theta^*_\leqk- \hat{\theta}(y)_\leqk(\phi(X)\theta^*)\|_{\Sigma_\leqk}^2$. By Lemma \eqref{lem:leqk_pred}, we have
\[
\hat{\theta}( \phi(X)\theta^*)_\leqk + \phi_\leqk(X)^\top A_k^{-1} \phi_\leqk(X)\hat{\theta}( \phi(X)\theta^*)_\leqk = \phi_\leqk(X)^\top A_k^{-1}\phi(X)\theta^*.
\]
Denote the error vector as $\zeta:= \hat{\theta}( \phi(X)\theta^*) - \theta^*$. We can rewrite the equation above as
\[
\zeta_\leqk + \phi_\leqk(X)^\top A_k^{-1} \phi_\leqk(X)\zeta_\leqk = \phi_\leqk(X)^\top A_k^{-1}\phi_\geqk(X) \theta^*_\geqk - \theta^*_\leqk.
\]

Multiplying both sides by $\zeta_\leqk^\top$ from the left and using that $\zeta_\leqk^\top \zeta_\leqk = \|\zeta_\leqk\|^2 \geq 0$ we obtain
\[
\zeta_\leqk^\top \phi_\leqk(X)^\top A_k^{-1} \phi_\leqk(X)\zeta_\leqk \leq \zeta_\leqk^\top \phi_\leqk(X)^\top A_k^{-1}\phi_\geqk(X) \theta^*_\geqk - \zeta_\leqk^\top \theta^*_\leqk.
\]

Next, divide and multiply by $ \Sigma_\leqk^{1/2}$ in several places: 
\begin{align*}
\zeta_\leqk^\top \Sigma_\leqk^{1/2} \psi_{\leqk}(X)^{\top} A_k^{-1} \psi_{\leqk}(X) \Sigma^{1/2}_\leqk\zeta_\leqk \leq& \zeta_\leqk^\top \Sigma_\leqk^{1/2} \psi_{\leqk}(X)^{\top} A_k^{-1}\phi_\geqk(X) \theta^*_\geqk\\
&\qquad - \zeta_\leqk^\top \Sigma_\leqk^{1/2} \Sigma^{-1/2}_\leqk \theta^*_\leqk.
\end{align*}
Now we pull out the lowest singular values of the matrices in the LHS and largest singular values of the matrices in the RHS to obtain lower and upper bounds respectively, yielding
\begin{align*}
& \|\zeta_\leqk\|_{\Sigma_\leqk}^2\mu_n(A_k^{-1})\mu_k\left(\psi_{\leqk}(X)^{\top} \psi_{\leqk}(X) \right) \\
& \leq  \|\zeta_\leqk\|_{\Sigma_\leqk} \mu_1(A_k^{-1})\sqrt{\mu_1\left(\psi_{\leqk}(X)^{\top} \psi_{\leqk}(X) \right)}\|\phi_\geqk(X) \theta^*_\geqk\|\\
& + \|\zeta_\leqk\|_{\Sigma_\leqk} \|\theta_\leqk^*\|_{\Sigma_\leqk^{-1}},
\end{align*}
and so
\begin{align*}
\|\zeta_\leqk\|_{\Sigma_\leqk} \leq& \frac{\mu_1(A_k^{-1})}{\mu_n(A_k^{-1})}\frac{\mu_1\left(\psi_{\leqk}(X)^{\top} \psi_{\leqk}(X) \right)^{1/2}}{\mu_k\left(\psi_{\leqk}(X)^{\top} \psi_{\leqk}(X) \right)}\|\phi_\geqk(X) \theta^*_\geqk\| \\ +&\frac{\|\theta_\leqk^*\|_{\Sigma_\leqk^{-1}}}{\mu_n(A_k^{-1})\mu_k\left(\psi_{\leqk}(X)^{\top} \psi_{\leqk}(X) \right)}.
\end{align*}

This completes the bounds for the $\leqk$ components and we now move on to the $\geqk$ ones. The contribution of the components of $\zeta$, starting from the $k+1$st can be bounded as follows:
\begin{multline*}
\|\theta^*_\geqk - \phi_\geqk(X)^\top A^{-1}\phi(X)\theta^*\|^2_{\Sigma_\geqk}\\
\leq 3\left(\|\theta^*_\geqk\|_{\Sigma_\geqk}^2 + \|\phi_\geqk(X)^\top A^{-1}\phi_\geqk(X)\theta^*_\geqk\|^2_{\Sigma_\geqk} + \|\phi_\geqk(X)^\top A^{-1}\phi_\leqk(X)\theta^*_\leqk\|^2_{\Sigma_\geqk}\right).
\end{multline*}

First of all, let's deal with the second term:
\begin{align*}
\|\phi_\geqk(X)^\top A^{-1}\phi_\geqk(X)\theta^*_\geqk\|^2_{\Sigma_\geqk}
=& \|\Sigma_\geqk^{1/2}\phi_\geqk(X)^\top A^{-1}\phi_\geqk(X)\theta^*_\geqk\|^2\\
\leq&\|\Sigma_\geqk\|\|\phi_\geqk(X)^\top A^{-1}\phi_\geqk(X)\theta^*_\geqk\|^2\\
=&\norm{\Sigma_{\geqk}}(\theta^*_\geqk)^\top \phi_\geqk(X)^\top A^{-1} \underset{=A - n\reg I - \phi_\leqk(X)\phi_\leqk(X)^\top \preceq A}{\underbrace{\phi_\geqk(X) \phi_\geqk(X)^\top}} A^{-1}\phi_\geqk(X)\theta^*_\geqk\\
\leq& \norm{\Sigma_{\geqk}}(\theta^*_\geqk)^\top \phi_\geqk(X)^\top A^{-1}\phi_\geqk(X)\theta^*_\geqk\\
\leq&\norm{\Sigma_{\geqk}} \mu_1(A_k^{-1}) \|\phi_\geqk(X)\theta^*_\geqk\|^2,
\end{align*}
where we used that $\mu_1(A_k^{-1}) \geq \mu_1(A^{-1})$ in the last transition.

Now, let's deal with the last term. Note that $A = A_k + \phi_\leqk(X) \phi_\leqk(X)^\top$. By the Sherman–Morrison–Woodbury formula,
\begin{align*}
A^{-1}\phi_\leqk(X) =& (A_k^{-1} + \phi_\leqk(X) \phi_\leqk(X)^\top)^{-1}\phi_\leqk(X)\\
=&\left(A_k^{-1} - A_k^{-1}\phi_\leqk(X)\left(I_k + \phi_\leqk(X)^\top A_k^{-1}\phi_\leqk(X)\right)^{-1}\phi_\leqk(X)^\top A_k^{-1}\right)\phi_\leqk(X)\\
=&A_k^{-1}\phi_\leqk(X)\left(I_n -  \left(I_k + \phi_\leqk(X)^\top A_k^{-1}\phi_\leqk(X)\right)^{-1}\phi_\leqk(X)^\top A_k^{-1}\phi_\leqk(X)\right)\\
=&A_k^{-1}\phi_\leqk(X)\left(I_n -  \left(I_k + \phi_\leqk(X)^\top A_k^{-1}\phi_\leqk(X)\right)^{-1}\left(I_k + \phi_\leqk(X)^\top A_k^{-1}\phi_\leqk(X) - I_k\right)\right)\\
=& A_k^{-1}\phi_\leqk(X)\left(I_k + \phi_\leqk(X)^\top A_k^{-1}\phi_\leqk(X)\right)^{-1}.
\end{align*}
Thus, 
\begin{align*}
&\|\phi_\geqk(X)^\top A^{-1}\phi_\leqk(X)\theta^*_\leqk\|^2_{\Sigma_\geqk}\\
=& \|\phi_\geqk(X)^\top A_k^{-1}\phi_\leqk(X)\left(I_k + \phi_\leqk(X)^\top A_k^{-1}\phi_\leqk(X)\right)^{-1}\theta^*_\leqk\|^2_{\Sigma_\geqk}\\
=& \|\Sigma_{\geqk}^{1/2}\phi_\geqk(X)^\top A_k^{-1}\psi_{\leqk}(X)\left(\Sigma_{\leqk}^{-1} + \psi_{\leqk}(X)^{\top} A_k^{-1}\psi_{\leqk}(X)\right)^{-1}\Sigma_{\leqk}^{-1/2}\theta^*_\leqk\|^2\\
\leq& \|A_k^{-1/2}\phi_\geqk(X) \Sigma_\geqk \phi_\geqk(X)^\top A_k^{-1/2}\|\mu_1(A_k^{-1/2})^2 \frac{\mu_1(\psi_{\leqk}(X)^{\top} \psi_{\leqk}(X))}{\mu_k(\psi_{\leqk}(X)^{\top} A_k^{-1}\psi_{\leqk}(X))^2}\|\Sigma_{\leqk}^{-1/2}\theta^*_\leqk\|^2\\
\leq&\|\Sigma_\geqk \|\|A_k^{-1/2}\phi_\geqk(X) \phi_\geqk(X)^\top A_k^{-1/2}\|\frac{\mu_1(A_k^{-1})}{\mu_n(A_k^{-1})^2} \frac{\mu_1(\psi_{\leqk}(X)^{\top} \psi_{\leqk}(X))}{\mu_k(\psi_{\leqk}(X)^{\top} \psi_{\leqk}(X))^2}\|\Sigma_{\leqk}^{-1/2}\theta^*_\leqk\|^2\\
=& \norm{\Sigma_{\geqk}}\|I_n - n\reg A_k^{-1}\|\frac{\mu_1(A_k^{-1})}{\mu_n(A_k^{-1})^2} \frac{\mu_1(\psi_{\leqk}(X)^{\top} \psi_{\leqk}(X))}{\mu_k(\psi_{\leqk}(X)^{\top} \psi_{\leqk}(X))^2}\|\Sigma_{\leqk}^{-1/2}\theta^*_\leqk\|^2\\
\leq& \norm{\Sigma_{\geqk}}\frac{\mu_1(A_k^{-1})}{\mu_n(A_k^{-1})^2} \frac{\mu_1(\psi_{\leqk}(X)^{\top} \psi_{\leqk}(X))}{\mu_k(\psi_{\leqk}(X)^{\top} \psi_{\leqk}(X))^2}\|\Sigma_{\leqk}^{-1/2}\theta^*_\leqk\|^2,\\
\end{align*}
where in the last transition we used the fact that $I_n - n\reg A_k^{-1}$ is a PSD matrix with norm bounded by 1 for $\reg > 0$.

Putting those bounds together yields the result.

\end{proof}

\section{Applications - Proofs of Results in \secref{sec:applications}}\label{app:applications}
\subsection{Regularized Case (\thmref{thm:fixed_dimensional})}
\strongreg*
\begin{proof}
   We use \thmref{thm:bound_gen}, which states that there exist some absolute constants $c,c'>0$ s.t for any $k\in\N$ with $c\beta_kk\log(k)\leq n$ and any $\delta>0$, \eqref{eq:bound_var} and \eqref{eq:bound_bias} hold w.p at least $1 - \delta - 16\exp\left(-\frac{c'}{\beta_k^2}\frac{n}{k}\right)$.

    In order to use the theorem, for any $n$ we first have to pick some $k\in\N$ s.t $c\beta_kk\log(k)\leq n$. As such, let $k:=k(n):= \left\lceil n^{\frac{1+b}{1+a}} \right\rceil$. The condition $b\in (-1,a)$ implies that $\frac{1+b}{1+a} < 1$, and thus $k(n) = o_n\left(\frac{n}{\log(n)}\right)$, meaning that for sufficiently large $n$, \thmref{thm:bound_gen} can be used with this chosen $k$. Since $k$ is a function of $n$, the $\bigo_n$ notation in particular, implies constants w.r.t $k$.

    We now proceed to bounding $\err_{k,n}$ (as defined in \thmref{thm:bound_gen}). By \lemref{lem:mu1_bound_poly} it holds w.p at least $1- \bigo_{n}\left(\frac{1}{k^3}\right)\exp\left(-\Omega_{n}(\frac{n}{k})\right)$ that
    \begin{align}\label{eq:mu1_bound_poly}
        \mu_1\left(\frac{1}{n}\krmat_{\geqk}\right) = \bigo_{n}\left(\lambda_{k+1}\right) = \bigo_n\left(\left(n^{\frac{1+b}{1+a}}\right)^{-(1+a)}\right) = \bigo_n\left(n^{-1-b}\right) = \bigo_n\left(\reg\right).
    \end{align}

    We can bound the event that both \thmref{thm:bound_gen} hold and \eqref{eq:mu1_bound_poly} hold as
    \[
    1 - \delta - 16\exp\left(-\frac{c'}{\beta_k^2}\frac{n}{k}\right) - \bigo_{n}\left(\frac{1}{k^3}\right)\exp\left(-\Omega_{n}\left(\frac{n}{k}\right)\right) = 1 - \delta - \bigo_n\left(\frac{1}{n}\right),
    \]
    Where we used the facts that $\frac{c'}{\beta_k^2}\frac{n}{k} = \omega_n\left(\log(n)\right)$.
    From now on, we assume that both \thmref{thm:bound_gen} and \eqref{eq:mu1_bound_poly} indeed hold. 
    Plugging \eqref{eq:mu1_bound_poly} into the definition of the concentration coefficient \eqref{eq:concentration} and using $\mu_n\left(\frac{1}{n}\krmat_{\geqk}\right)\geq 0$, we obtain the bound
    \begin{align}\label{eq:conc_poly}
        \err_{k,n} = \bigo_{n}\left(\frac{\lambda_{k+1} + \reg}{\reg}\right) =\bigo_n\left(\frac{\reg}{\reg}\right) = \bigo_n\left(1\right).
    \end{align}

    By \lemref{lem:int_dim}, it holds that $r_k(\Sigma), r_k(\Sigma^2) = \Theta_{n}\left(k\right)$. So plugging this and \eqref{eq:conc_poly} into \thmref{thm:bound_gen} yields, 
    \[
    \nicefrac{V}{\sigma_{\epsilon}^2} =\bigo_{n}\left(\frac{k}{n} + \frac{r_k(\Sigma^2)}{n}\right) =\bigo\left(\frac{k}{n}\right) =\bigo_n\left(\frac{n^{\frac{1+b}{1+a}}}{n}\right) = \bigo_n\left(n^{\frac{b-a}{1+a}}\right), 
    \]
    and 
    \begin{align*}
        B = & \frac{1}{\delta}\bigo_n\left(\underset{:=T_1}{\underbrace{\norm{\theta^*_\geqk}_{\Sigma_\geqk}^2}}\right) + \bigo_n\left(\underset{:=T_2}{\underbrace{\norm{\theta_\leqk^*}_{\Sigma_\leqk^{-1}}^2}} \underset{:=T_3}{\underbrace{\left(\reg + \frac{\tr\left(\Sigma_{\geqk}\right)}{n}\right)^2}}\right).
    \end{align*}
    Following \lemref{lem:int_dim} it holds that $\tr(\Sigma_{\geqk})=\bigo_{n}(k\cdot \lambda_k) = \bigo_{n}(k\cdot \reg)$ and so 
    \begin{align*}
        T_3=\bigo_{n}\left(\left(\reg + \frac{k}{n}\reg\right)^2\right) = \bigo_{n}\left(\reg^2\right) = \bigo_{n}\left(\frac{1}{n^{2+2b}}\right).
    \end{align*}
    Combining this bound for $T_3$ with the bounds for $T_1,T_2$ from \lemref{lem:bound_bias_poly} yields
    \begin{align*}
        B \leq &  
        \begin{cases}
            \bigo_{n}\left(\frac{1}{k^{2r+a}} + \frac{1}{k^{2r-2-a} n^{2(1+b)}}\right) & 2r < 2 + a \\
            \bigo_{n}\left(\frac{1}{k^{2(1+a)}} + \frac{\log(k)}{n^{2(1+b)}}\right) & 2r = 2 + a \\
            \bigo_{n}\left(\frac{1}{k^{2r+a}} + \frac{1}{n^{2(1+b)}}\right) & 2r > 2 + a
        \end{cases} \\ 
        \leq &        
        \begin{cases}
            \bigo_{n}\left(\frac{1}{n^{\frac{(2r+a)(1+a)}{1+b}}} \right) & 2r < 2 + a \\
            \bigo_{n}\left(\frac{\log(n)}{n^{2(1+b)}}\right) & 2r = 2 + a \\
            \bigo_{n}\left(\frac{1}{n^{2(1+b)}}\right) & 2r > 2 + a
        \end{cases}.
    \end{align*}
 
\end{proof}

\subsection{Fixed Dimensional Interpolation Case (\thmref{thm:min_norm_poly})}
\minnormpoly*

\begin{proof}
    We use \thmref{thm:bound_gen}, which states that there exist some absolute constants $c,c'>0$ s.t for any $k\in\N$ with $c\beta_kk\log(k)\leq n$ and any $\delta>0$, \eqref{eq:bound_var} and \eqref{eq:bound_bias} hold w.p at least $1 - \delta - 16\exp\left(-\frac{c'}{\beta_k^2}\frac{n}{k}\right)$.

    In order to use the theorem, for any $n$ we first have to pick some $k\in\N$ s.t $c\beta_kk\log(k)\leq n$. Using the fact that $\beta_k\leq C_{0}$ for some $C_0>0$, let $k:=k(n):= \frac{n}{\max(cC_0, 1)\log(n)}$ and we also let $k':=k'(n)=n^2\log^4(n)$. The probability that \thmref{thm:bound_gen} holds with $k(n)$ now becomes $1 - \delta - \bigo_n(\frac{1}{n})$. Since $k$ is a function of $n$, the $\bigo_n$ notation in particular, implies constants w.r.t $k$.
    
    In order to bound \eqref{eq:bound_var} and \eqref{eq:bound_bias}, we begin by bounding $\err_{k,n}$, which requireds bounding $\mu_1\left(\frac{1}{n}\krmat_{\geqk}\right)$ and $\mu_n\left(\frac{1}{n}\krmat_{\geqk}\right)$. 
    First note that by \citet{bartlett2020benign}[Lemma 5] $R_k \geq r_k$ and thus by \lemref{lem:int_dim} it holds that $R_{k'} = \Omega_n\left(n^2\log^4(n)\right)$ and $\tr\left(\Sigma_{\geqk'}\right) = \Omega_n\left((n^2\log^4(n))^{-a}\right)$. By \corref{cor:ak_eigen_bound} it holds w.p at least $1-\frac{1}{\log(n)}$ that, 
    \begin{align}\label{eq:poly_lower}
        \mu_n\left(\frac{1}{n}\krmat_{\geqk}\right) \geq & \alpha_k\left(1-\frac{1}{\log(n)}\sqrt{\frac{n^2}{R_{k'}}}\right) 
        \frac{\tr\left(\Sigma_{\geqk'}\right)}{n} \nonumber\\
        = & \Omega_n\left(\left(1-\log(n)\sqrt{\frac{1}{\log^4(n)}}\right)\frac{\tr\left(\Sigma_{\geqk'}\right)}{n}\right) \nonumber\\ 
        = & \Omega_n\left(\frac{(n^2\log^4(n))^{-a}}{n}\right) = \Omega_n\left(n^{-1-2a}\log^{-4a}(n)\right).
    \end{align}
    For $\mu_1\left(\frac{1}{n}\krmat\right)$, by \lemref{lem:mu1_bound_poly} it holds w.p at least $1- \bigo_{n}\left(\frac{1}{k^3}\right)\exp\left(-\Omega_{n}\left(\frac{n}{k}\right)\right)$ that
    \begin{align}\label{eq:mu1_bound_poly2}
        \mu_1\left(\frac{1}{n}\krmat_{\geqk}\right) = \bigo_{n}\left(\lambda_{k+1}\right) = \bigo_n\left(n^{-1-a}\log^{1+a}(n)\right).
    \end{align}

    So \thmref{thm:bound_gen}, \eqref{eq:poly_lower} and \eqref{eq:mu1_bound_poly2} all hold simultaneously with probability $1-\delta - \bigo_n\left(\frac{1}{\log(n)}\right)$, and from now on we assume that this is indeed the case.

    By combining \eqref{eq:mu1_bound_poly2} and \eqref{eq:poly_lower} we obtain the bound 
    \begin{align}\label{eq:poly_conc}
        \err_{k,n} = \bigo_n\left(\frac{n^{-1-a}\log^{1+a}(n)}{n^{-1-2a}\log^{-4a}(n)}\right) = \tilde{\bigo}_n\left(n^{a}\right)
    \end{align}

    And thus by combining \eqref{eq:bound_var}, \eqref{eq:poly_conc} and the fact that from \lemref{lem:int_dim} $r_k(\Sigma^2)\lesssim k$, we obtain the bound 
    \begin{align*}
        \nicefrac{V}{\sigma_{\epsilon}^2} = \tilde{\bigo}_n\left(n^{2a}\frac{k}{n}\right) = \tilde{\bigo}_n\left(n^{2a}\right)
    \end{align*}
    and 
    \begin{align*}
        B = & \frac{1}{\delta}\tilde{\bigo}_n\left( n^{3a}\left(\underset{:=T_1}{\underbrace{\norm{\theta^*_\geqk}_{\Sigma_\geqk}^2}} + \underset{:=T_2}{\underbrace{\norm{\theta_\leqk^*}_{\Sigma_\leqk^{-1}}^2}} \underset{:=T_3}{\underbrace{\left(\frac{\tr\left(\Sigma_{\geqk}\right)}{n}\right)^2}}\right)\right).
    \end{align*}
    Following \lemref{lem:int_dim} it holds that $\tr(\Sigma_{\geqk})=\tilde{\bigo}_n(k\cdot \lambda_k) = \tilde{\bigo}_n( \frac{1}{n^a})$ and so $T_3=\tilde{\bigo}_n\left(\frac{1}{n^{2+2a}}\right)$.
    Combining this bound for $T_3$ with the bounds for $T_1,T_2$ from \lemref{lem:bound_bias_poly} yields
    \begin{align*}
    T_1 + T_2T_3 \leq &
    \begin{cases}
        \\ \tilde{\bigo}_n\left(\frac{1}{n^{2r+a}} + \frac{1}{n^{2r-2-a}n^{2(1+a)}}\right) & 2r \leq 2 + a \\
        \tilde{\bigo}_n\left(\frac{1}{n^{2r+a}} + \frac{1}{n^{2(1+a)}}\right) & 2r > 2 + a
    \end{cases} \\ 
    =& \tilde{\bigo}_n\left(\frac{1}{n^{\min\left(2r+a, 2(1+a)\right)}}\right).
    \end{align*}

    Implying that
    \begin{align*}
        B \leq & \frac{1}{\delta}\tilde{\bigo}_n\left( n^{3a}\frac{1}{n^{\min\left(2r+a, 2(1+a)\right)}}\right) \\ 
        = & \frac{1}{\delta} \tilde{\bigo}_n\left(\frac{1}{\min\left(n^{2(r-a), 2-a}\right)}\right).
    \end{align*}
\end{proof}

\subsection{High Dimensional Interpolation Case (\thmref{thm:highdim})}
\highdim*

\begin{proof}
Let $\sigma_{\ell}:=\frac{\hat{\sigma}_\ell}{N(d,\ell)}$ be the eigenvalues from \eqref{eq:mercer-dot-product}. We order $\phi$ in the natural way, by first taking $\tilde{\phi}(\x)=(\sqrt{\sigma_0}Y_{0,1}, \sqrt{\sigma_1}Y_{1,1},\ldots,\sqrt{\sigma_1}Y_{1,N(d,1)}, \sqrt{\sigma_2}Y_{2,1}, \ldots )$, and letting $\phi$ be the same as $\tilde{\phi}$ with zero-valued indices removed (where $\sigma_{\ell}=0$). We let $\psi$ be given accordingly.

For any $s\in\N$, and $d\in\N$ let $k_s(d)=\sum_{\ell=0}^{s}N(d,\ell) \cdot \mathbb{I}_{\hat{\sigma}_{\ell}}$ where 
$\mathbb{I}_{\hat{\sigma}_{\ell}}=\begin{cases}
    1 & \hat{\sigma}_{\ell} > 0 \\
    0 & \text{else}
\end{cases}$. 
Let $\Delta_{\geqk_s(d)}\in\R^{n\times n}$ be the matrix given by $[\Delta_{\geqk_s(d)}]_{ij} = \begin{cases}
    \frac{1}{n}[\krmat_{\geqk_s(d)}]_{ij} & i\neq j \\
    0 & i=j
\end{cases}$. 
By \eqref{eq:helper_e_k} we have that
\begin{align}\label{eq:helper_e_k_copy}
        \alpha_{k_s(d)}\frac{1}{n}\tr\left(\Sigma_{\geqk_s(d)}\right) + \mu_{n}\left(\Delta_{\geqk_s(d)}\right) \leq \mu_i\left(\frac{1}{n}\krmat_{\geqk_s(d)}\right)\leq \beta_{k_s(d)}\frac{1}{n}\tr\left(\Sigma_{\geqk_s(d)}\right) + \mu_{1}\left(\Delta_{\geqk_s(d)}\right).
\end{align}

In order to bound the eigenvalues of $\Delta_{\geqk_s(d)}$ we will need to control the effective ranks. Let $j_s:=\argmax_{j\geq s} \sigma_j$, then 
\[
r_{k_s(d)}(\Sigma) = \frac{\sum_{i=s+1}^{\infty}N(d,i)\sigma_i}{\sigma_{j_{s+1}}} \geq N(d, j_{s+1}) \geq N(d, s+1),
\]
where our assumption that $\hat{\sigma}_{\ell}>0$ for some $\ell \geq \lfloor 2\tau \rfloor$ ensures that $\sigma_{j_{s+1}}>0$ for $s\leq \lfloor 2\tau \rfloor$.
By \citet{bartlett2020benign}[Lemma 5] we also have $R_{k_s(d)}(\Sigma) \geq r_{k_s(d)}(\Sigma)$. Let $k(d):= k_{\lfloor \tau \rfloor}(d)$ and $v(d):= k_{\lfloor 2\tau \rfloor}(d)$. Let $t=\min(\lfloor\tau\rfloor - \tau + 1, \lfloor2\tau\rfloor - 2\tau + 1) > 0$, then by what we just showed, and using the fact that for any $i\in\N$, $N(d,i)=\Theta_d\left(d^i\right)$, we have the following identities:
\begin{align}
    R_{v(d)}(\Sigma) \geq & \Omega_{n,d}\left(d^{\lfloor2\tau\rfloor+1}\right) \geq  \Omega_{n,d}\left(n^{2 + \frac{t}{\tau}}\right), \\
    r_{k(d)}(\Sigma) \geq &  \Omega_{n,d}\left(d^{\lfloor\tau\rfloor+1}\right) \geq  \Omega_{n,d}\left(n^{1 + \frac{t}{\tau}}\right).
\end{align}
We have shown that conditions (A2) and (A3) of \citet{mei2022generalization}[Proposition 4] hold. Furthermore, condition (A1) holds applying \citet{mei2022generalization}[Lemma 19] to $\psi_{\leq v(d)}$. As a result, \citet{mei2022generalization}[Proposition 4] states that for some $t'>0$, 
\[
\norm{\Delta_{\geq k(d)}} \leq \bigo_{n,d}\left(d^{-t'}\right) \cdot \frac{1}{n}\tr\left(\Sigma_{\geqk(d)}\right).
\]
Plugging this into \eqref{eq:helper_e_k_copy} and using that by the addition theorem \eqref{eq:addition_1}, $\alpha_{k(d)}=\beta_{k(d)}=1$, it holds that
\begin{align}\label{eq:K_geqk_bound}
    \mu_i\left(\frac{1}{n}\krmat_{\geqk(d)}\right) = \Theta_{n,d}\left(\frac{1}{n}\tr\left(\Sigma_{\geqk(d)}\right)\right).
\end{align}

As a result, we obtain that for $\err_{k,n}$ as defined in \thmref{thm:bound_gen}, 
\begin{align*}
    \err_{k(d),n} = \frac{\norm{\Sigma_{\geqk}} + \mu_1\left(\frac{1}{n}\krmat_{\geqk}\right) + \reg}{\mu_n\left(\frac{1}{n}\krmat_{\geqk}\right) + \reg} = \bigo_{n,d}\left(\frac{\left(\frac{n}{r_{k(d)}} + 1\right)\frac{1}{n}\tr\left(\Sigma_{\geqk(d)}\right)}{\frac{1}{n}\tr\left(\Sigma_{\geqk(d)}\right)}\right) \leq \bigo_{n,d}\left(1\right).
\end{align*}
Combining this with \thmref{thm:bound_gen}, it holds that for every $\delta>0$, w.p at least $1-\delta - 16\exp\left(-\frac{c'}{\beta_{k(d)}^2}\frac{n}{k(d)}\right)$, both the variance and bias can be upper bounded as
\begin{align}
    V \leq & \sigma_\epsilon^2 \bigo_{n,d}\left(\left(\frac{k(d)}{n} + \frac{n}{R_k(\Sigma)}\right)\right) \leq \sigma_{\epsilon}^2 \bigo_{n,d}\left(\frac{1}{d^{\tau - \lfloor\tau \rfloor}} + \frac{1}{d^{\lfloor\tau\rfloor + 1 - \tau}}\right).
\end{align}
\begin{align}\label{eq:bias_inter}
        B \leq \frac{1}{\delta} \bigo_{n,d}\left(\norm{\theta^*_{\geqk(d)}}_{\Sigma_{\geqk(d)}}^2\right) + \bigo_{n,d}\left(\norm{\theta_{\leqk(d)}^*}_{\Sigma_{\leqk(d)}^{-1}}^2 \left(\frac{\tr\left(\Sigma_{\geqk(d)}\right)}{n}\right)^2\right).
\end{align}
Using the fact that $\frac{c'}{\beta_{k(d)}^2}\frac{n}{k(d)} = \omega_d(\log(d))$ the probability becomes $1 - \delta - o_d\left(\frac{1}{d}\right)$

Now in order to further bound the bias, we first note that by the addition theorem \eqref{eq:addition_1} it holds that
\begin{align}\label{eq:o1_trace}
    \tr\left(\Sigma\right) = \sum_{\ell=0}^\infty \sigma_{\ell}N(d,\ell) = h(1) = \Theta_{n,d}(1).
\end{align}
As in the statement of the lemma, let $N_d:=k(d)$. Because $i\in\N$, $N(d,i)=\Theta_d\left(d^i\right)$ and by assumption, $\hat{\sigma}_{\lfloor\tau \rfloor}\neq 0$, it holds that $k(d)=\bigo_{n,d}\left(d^{\lfloor \tau \rfloor}\right)$. Combining this with \eqref{eq:o1_trace} and the fact that for all $i\leq k(d)$, $\lambda_i \geq \underset{\ell \leq \lfloor \tau \rfloor \st \hat{\sigma}_{\ell}\neq 0}{\min} \hat{\sigma}_{\ell} \cdot \Omega_{n,d}\left(\frac{1}{d^{\lfloor\tau \rfloor}}\right)$ the right hand side of \eqref{eq:bias_inter} can be bounded as
\begin{align*}
    \norm{\theta_{\leqk(d)}^*}_{\Sigma_{\leqk(d)}^{-1}}^2 \left(\frac{\tr\left(\Sigma_{\geqk(d)}\right)}{n}\right)^2 = & \sum_{i\leqk(d)} \frac{(\theta_i^*)^2}{\lambda_i} \left(\frac{\tr\left(\Sigma_{\geqk(d)}\right)}{n}\right)^2 \\
    \leq & \frac{k(d)}{\min_{i\leqk(d)} \lambda_i} \norm{\theta^*_{\leq N_d}}_{\infty}^2 \left(\frac{\tr\left(\Sigma\right)}{n}\right)^2 \\
    \leq &  \norm{\theta^*_{\leq N_d}}_{\infty}^2 \frac{1}{\underset{\ell \leq \lfloor \tau \rfloor \st \hat{\sigma}_{\ell}\neq 0}{\min} \hat{\sigma}_{\ell}} \cdot \bigo_{n,d}\left(\frac{1}{d^{2(\tau - \lfloor\tau \rfloor})}\right).
\end{align*}
The left hand side of \eqref{eq:bias_inter} can be bounded as
\begin{align*}
    \frac{1}{\delta}\norm{\theta^*_{\geqk(d)}}_{\Sigma_{\geqk(d)}}^2 \leq \frac{1}{\delta}\norm{\theta^*_{\geqk(d)}}_{\infty}^2 \tr\left(\Sigma_{\geqk(d)}\right) = \frac{1}{\delta}\bigo_{n,d}\left(\norm{\theta^*_{> N_d}}_{\infty}^2\right).
\end{align*}

So \eqref{eq:bias_inter} becomes
\[
B \leq \frac{1}{\delta}\bigo_{n,d}\left(\norm{\theta^*_{> N_d}}_{\infty}^2\right) + \norm{\theta^*_{\leq N_d}}_{\infty}^2 \underset{\ell \leq \lfloor \tau \rfloor \st \hat{\sigma}_{\ell}\neq 0}{\max} ~ \frac{1}{\hat{\sigma}_{\ell}} \cdot \bigo_{n,d}\left(\frac{1}{d^{2(\tau - \lfloor\tau \rfloor)}}\right). 
\]

\end{proof}

\subsection{Lemmas for Applications}
\begin{restatable}{lemma}{intdim}\label{lem:int_dim}
    For any $a>0$, 
    \begin{enumerate}
        \item If $c_1 \frac{1}{i\log^{1+a}(i)} \leq \lambda_i \leq c_2\frac{1}{i\log^{1+a}(i)}$ then $\frac{c_1}{c_2}\frac{1}{a} (k+1)\log(k+1) \leq r_k\leq 1 + \frac{c_2}{c_1}\frac{1}{a} (k+1)\log(k+1)$.
        \item If $c_1 \frac{1}{i^{1+a}} \leq \lambda_i \leq c_2\frac{1}{i^{1+a}}$ then $\frac{c_1}{c_2}\frac{1}{a} (k+1)\leq r_k\leq 1 + \frac{c_2}{c_1}\frac{1}{a} (k+1)$.
        \item If $c_1 \frac{1}{e^{ai}} \leq \lambda_i \leq c_2\frac{1}{e^{ai}}$ then $\frac{c_1}{c_2}\frac{1}{a} \leq r_k\leq 1 + \frac{c_2}{c_1}\frac{1}{a}$.
    \end{enumerate}
\end{restatable}
\begin{proof}
The famous integral test for convergence states that for a monotonic decreasing function $f(n)$, it holds for any $k\in\N$ that
\begin{align*}\label{eq:int_test}
    \int_{k+1}^\infty f(x)dx \leq \sum_{i\geqk}f(i) \leq f(k+1) + \int_{k+1}^\infty f(x)dx,
\end{align*}

We now split into separate cases of eigenvalue decay.
    \begin{enumerate}
        \item If $c_1 \frac{1}{i\log^a(i)} \leq \lambda_i \leq c_2\frac{1}{i\log^a(i)}$ then using the fact that $\int_{k+1}^\infty \frac{1}{x\log^{1+a}(x)}dx = \frac{1}{a\log^a(k+1)}$ we obtain
        \begin{align*}
            r_k \leq 1 + \frac{1}{c_1\lambda_{k+1}}\int_{k+1}^\infty c_2\frac{1}{x\log^{1+a}(x)}dx \leq 1 + \frac{c_2}{c_1} \frac{1}{a}(k+1)\log(k+1),
        \end{align*}
        and 
        \begin{align*}
            r_k \geq \frac{1}{c_2\lambda_{k+1}}\int_{k+1}^\infty c_a\frac{1}{x\log^{1+a}(x)}dx \geq \frac{c_1}{c_2} \frac{1}{a}(k+1)\log(k+1).
        \end{align*}
            
        \item If $c_1 \frac{1}{i^{1+a}} \leq \lambda_i \leq c_2\frac{1}{i^{1+a}}$ then using the fact that $\int_{k+1}^\infty \frac{1}{x^{1+a}(x)}dx = \frac{1}{a(k+1)^a}$ we obtain that
        \begin{align*}
            r_k \leq 1 + \frac{1}{c_1\lambda_{k+1}}\int_{k+1}^\infty c_2\frac{1}{x^{1+a}(x)}dx \leq 1 + \frac{c_2}{c_1} \frac{1}{a}(k+1),
        \end{align*}
        and
        \begin{align*}
            r_k \geq \frac{1}{c_2\lambda_{k+1}}\int_{k+1}^\infty c_1\frac{1}{x^{1+a}(x)}dx \geq \frac{c_1}{c_2} \frac{1}{a}(k+1).
        \end{align*}
            
        \item If $c_1 \frac{1}{e^{ai}} \leq \lambda_i \leq c_2\frac{1}{e^{ai}}$  then using the fact that $\int_{k+1}^\infty \exp(-ax)dx = \frac{1}{ae^{a(k+1)}}$ we obtain that
        \begin{align*}
            r_k \leq 1 + \frac{1}{c_1\lambda_{k+1}}\int_{k+1}^\infty c_2\exp(-ax)dx \leq 1 + \frac{c_2}{c_1} \frac{1}{a},
        \end{align*}
        and
        \begin{align*}
            r_k \geq \frac{1}{c_2\lambda_{k+1}}\int_{k+1}^\infty c_1\exp(-ax)dx \geq \frac{c_1}{c_2} \frac{1}{a}.
        \end{align*}
    \end{enumerate}
\end{proof}

\begin{lemma} \label{lem:mu1_bound_poly}
    Let $K$ be a kernel with polynomially decaying eigenvalues $\lambda_i=\Theta_{i,n}(i^{-1-a})$ for some $a>0$. Furthermore, suppose that $\frac{\beta_k k\log(k)}{n}=\bigo_{k,n}(1)$ and that $\beta_k=\bigo_k(1)$. Then it holds w.p at least $1- \bigo_{k,n}\left(\frac{1}{k^3}\right)\exp\left(-\Omega_{k,n}(\frac{n}{k})\right)$ that
    \[
    \mu_1\left(\frac{1}{n}\krmat_{\geqk}\right) = \bigo_{k,n}\left(\lambda_{k+1}\right)
    \]
\end{lemma}
\begin{proof}
    By \lemref{lem:int_dim}, it holds that $r_k(\Sigma), r_k(\Sigma^2) = \Theta_{k,n}\left(k\right)$. Now using \corref{cor:ak_eigen_bound} (note that \assref{assumption:good_beta} holds since $\beta_k=\bigo_k(1)$), there exist absolute constants $c,c'>0$ s.t it holds w.p at least $1-4 \frac{r_k}{k^4}\exp\left(-\frac{c'}{\beta_k}\frac{n}{r_k}\right)$ that
    \begin{align*}
        \mu_1\left(\frac{1}{n}\krmat_{\geqk}\right) \leq & c\left(\lambda_{k+1} + \beta_k\log(k+1)\frac{\tr\left(\Sigma_{\geqk}\right)}{n}\right) \nonumber\\ 
        = & \bigo_{k,n}\left(\lambda_{k+1} \left(1 + \beta_k\log(k+1)\frac{r_k}{n}\right)\right) \nonumber\\
        = & \bigo_{k,n}\left(\lambda_{k+1} \left(1 + \frac{\beta_kk\log(k)}{n}\right)\right) 
        = \bigo_{k,n}\left(\lambda_{k+1}\right).
    \end{align*}
    Now to bound the probability which this holds, we use the fact that $r_k=\Theta_{k,n}(k)$ together with the fact that $\exp\left(-\frac{c'}{\beta_k}\frac{n}{r_k}\right) < 1$ to get the claim holds w.p at least $1-4 \frac{r_k}{k^4}\exp\left(-\frac{c'}{\beta_k}\frac{n}{r_k}\right) = 1- \bigo_{k,n}\left(\frac{1}{k^3}\right)\exp\left(-\Omega_{k,n}(\frac{n}{k})\right)$.
\end{proof}

\begin{lemma}\label{lem:sum_leqk}
    Let $a\in \R$, $1<k\in\N$, then
    \begin{align*} 
    \sum_{i\leqk} i^{-a} \leq
        \begin{cases}
            1 + k^{1-a} & a < 1 \\
            1 + \log(k) & a = 1 \\
            \frac{1}{a-1} & a > 1
        \end{cases}
    \end{align*}
\end{lemma}
\begin{proof}
    If $a<0$, then bounding the mean with the maximum yields $\sum_{i\leqk} i^{-a}\leq k\cdot k^{-a}=k^{1-a}$. Next, if $a\neq 1$, bounding the sum with the integral yields
    \begin{align*}
        \sum_{i\leqk} i^{-a} \leq 1 + \int_{1}^k \frac{1}{x^a} dx = 1 + \frac{1}{a-1} - \frac{k^{1-a}}{a-1}.
    \end{align*}
    So if $a<1$, we obtain a $ 1 + k^{1-a}$ bound, and if $a>1$, a $1 + \frac{1}{a-1}$ bound. Lastly, if $a=1$ then we can similarly bound as
    \begin{align*}
        \sum_{i\leqk} i^{-a} \leq 1 + \int_{1}^k \frac{1}{x} dx = 1 + 1 + \log(k).
    \end{align*}
\end{proof}

\begin{lemma} \label{lem:bound_bias_poly}
    Let $1<k\in\N$ and suppose that $\lambda_i= \Theta_{i,n}\left(\frac{1}{i^{1+a}}\right)$ for some $a>0$, and $\theta_i^* = \Theta_{i,n}\left(i^{-r}\right)$ for some $r\in\R$ s.t $f^*\in L^2_{\mu}(\Xcal)$. It holds that
    \begin{align*}
        \norm{\theta^*_\geqk}_{\Sigma_\geqk}^2 \leq \bigo_{k,n}\left(\frac{1}{k^{2r+a}}\right), 
        \qquad\qquad
        \norm{\theta_\leqk^*}_{\Sigma_\leqk^{-1}}^2 \leq \begin{cases}
            \bigo_{k,n}\left(k^{-2r + 2 + a}\right) & 2r < 2 + a \\
            \bigo_{k,n}\left(\log(k)\right) & 2r = 2 + a \\
            \bigo_{k,n}\left(1\right) & 2r > 2 + a
        \end{cases}.
    \end{align*}

\end{lemma}
\begin{proof}
    The condition that $f^*\in L^2_{\mu}(\Xcal)$ implies $\sum_{i=1}^\infty \theta^* \lambda_i^2 = \norm{\langle\theta^* \phi(\x)\rangle} < \infty$.
    The $\geqk$ part can be bounded using \lemref{lem:int_dim} as
    \begin{align*}
          \norm{\theta_\geqk^*}_{\Sigma_\geqk}^2 =& \sum_{i\geqk}(\theta_i^*)^2\lambda_i = \bigo_{k,n}\left(\sum_{i\geqk}i^{-2r-1-a}\right) \leq \bigo_{k,n}\left(\frac{1}{k^{2r+a}}\right).
    \end{align*}
    The $\leqk$ part can be bounded using lemma \lemref{lem:sum_leqk} (with $2r-1-a$) as
    \begin{align*}
        \norm{\theta_\leqk^*}_{\Sigma_\leqk^{-1}}^2 =& \sum_{i\leqk}\frac{(\theta_i^*)^2}{\lambda_i} = \bigo_{k,n}\left(\sum_{i\leqk}i^{-2r+1+a}\right) \leq             
        \begin{cases}
            \bigo_{k,n}\left(k^{-2r + 2 + a}\right) & 2r < 2 + a \\
            \bigo_{k,n}\left(\log(k)\right) & 2r = 2 + a \\
            \bigo_{k,n}\left(1\right) & 2r > 2 + a
        \end{cases}.
    \end{align*}

\end{proof}

\section{Lack of Sub Gaussianity}\label{appendix:rbf}
Suppose our inputs are one-dimensional standard Gaussians $x\sim \Ncal(0,\sigma^2)$ and let $K(x,y)=\exp\left(-\gamma(x-y)^2\right)$ be the Gaussian (RBF) kernel. Such kernels have known Mercer decompositions \citep{fasshauer2011positive}, and if we pick for simplicity $\sigma=1$ and $\gamma=\frac{3}{8}$ (meaning that in their notation, $\alpha=\frac{1}{\sqrt{2}}$ and $\epsilon=\sqrt{\frac{3}{8}}$) we obtain that $\psi(x)=(\psi_i(x))_{i=0}^\infty$ is given by:
\begin{align}
    \psi_i(x)=\frac{\sqrt[4]{2}}{\sqrt{2^i i!}}e^{-\frac{x^2}{4}}H_i(x),
\end{align}
where $H_i(x)=(-1)^ie^{x^2}\frac{d^i}{dx^i}e^{-x^2}$ is the $i$'th order (physicist's) Hermite polynomial. Note that in this chapter, for ease of notation, we start counting at $i=0$.

Recall that a vector $Y$ is said to be sub-Gaussian if 
\[
\sup_{u:\norm{u}=1}\sup_{p\geq 1} \frac{1}{\sqrt{p}}\left(\E\left[\abs{\langle u, Y\rangle}^p\right]\right)^{\nicefrac{1}{p}} < \infty.
\]

In particular, taking $Y=\psi$ and $u=e_i$ we get that:
\begin{align}\label{eq:exp_psi_k}
\E\left[\abs{\langle u, Y\rangle}^p\right] = & \frac{1}{\sqrt{2\pi}} \int_{-\infty}^\infty \abs{\psi_i(x)}^p e^{-\frac{x^2}{2}}dx \nonumber\\
= & \frac{2^{\frac{p}{4} - \frac{1}{2}}}{\sqrt{\pi}\left(2^ii!\right)^{\nicefrac{p}{2}}} \int_{-\infty}^\infty \abs{H_i(x)}^p \exp\left(-\left(\frac{p}{4}+\frac{1}{2}\right)x^2\right) dx
\end{align}

Thus, if for a fixed $p$, The value of \eqref{eq:exp_psi_k} diverges to infinity with $i$, it would imply that $\psi$ is not sub-Gaussian.

We will thus aim to lower bound this term. To do so, we begin by bounding the Hermite polynomials using \citet{szeg1939orthogonal}[Theorem 8.22.9], which states that for any $\delta >0$, and any $x=\sqrt{2i+1}\cos(\phi)$ where $\delta\leq \phi \leq \pi-\delta$, we have the uniform approximation:
\begin{align}\label{eq:approx_raw}
e^{-\frac{x^2}{2}}H_i(x)= & 2^{\frac{i}{2}+\frac{1}{4}}\sqrt{i!}(\pi i)^{-\frac{1}{4}} \nonumber\\ 
& \times \underset{:=A}{\underbrace{\sin(\phi)^{-\frac{1}{2}}}}\left(\underset{:=B}{\underbrace{\sin\left(\frac{3\pi}{4} + \left(\frac{2i+1}{4}\right)(\sin(2\phi)-2\phi)\right)}} + \bigo(i^{-1})\right).
\end{align}
We now wish to bound $B$. Since $\sin(\phi) \geq 0.5$ for $\phi \in [\frac{1}{6}\pi, \frac{5}{6}\pi]$ then we can lower bound $B$ by $0.5$ when 
\[
\frac{3\pi}{4} + \left(\frac{2i+1}{4}\right)(\sin(2\phi)-2\phi)\in \left[\frac{1}{6}\pi, \frac{5}{6}\pi\right].
\]
This is equivalent to:
\[
-\frac{1}{6(2i+1)}\pi \leq \phi - \frac{\sin{2\phi}}{2} \leq \frac{7}{6(2i+1)}\pi.
\]
Since $\phi\geq 0$, we have (via the $\sin$ Taylor expansion) that $\phi - \frac{\phi^3}{6}\leq\sin(\phi)\leq\phi$ (meaning $-\phi\leq-\frac{\sin2\phi}{\phi}\leq -\phi + \frac{8\phi^3}{6}$) and so the lower bound holds trivially and the upper bound holds when $\phi \leq \sqrt[3]{\frac{7}{8(2i+1)}\pi}$. 

We can also lower bound $A$ trivially by $1$. Furtheremore, for $i$ sufficiently large the $\bigo(i^{-1})$ is at least $-\frac{1}{4}$. 
So overall we obtain that for $\phi\in\left[\delta, \sqrt[3]{\frac{7}{8(2i+1)}\pi}\right]$ and $x=\sqrt{2i+1}\cos(\phi)$, $A(B + \bigo(i^{-1}))\geq \frac{1}{4}$, and \eqref{eq:approx_raw} can be lower bounded as:
\[
H_i(x) \geq \frac{1}{4}2^{\frac{i}{2}+\frac{1}{4}}\sqrt{i!}(\pi i)^{-\frac{1}{4}} e^{\frac{x^2}{2}} = \frac{1}{4}\left(\frac{2}{\pi}\right)^{\frac{1}{4}} 2^{\frac{i}{2}}\sqrt{i!}i^{-\frac{1}{4}} e^{\frac{x^2}{2}}.
\]
So for any $p\in\N$, we can lower bound the $p$'th power of $H_i$ as 
\[
H_i(x)^p \geq \frac{1}{4^p} \left(\frac{2}{\pi}\right)^{\frac{p}{4}} \left(2^ii!\right)^{\nicefrac{p}{2}}i^{-\frac{p}{4}} e^{\frac{px^2}{2}}.
\]
Denoting $a_i=\sqrt{2i+1} \cos\left(\sqrt[3]{\frac{7}{8(2i+1)}\pi}\right)$ and $b_i=\sqrt{2i+1}\cos(\delta)$ we can bound our expected value in $\eqref{eq:exp_psi_k}$ by:
\begin{align*}
\E\left[\abs{\langle u, Y\rangle}^p\right] = & \frac{2^{\frac{p}{4} - \frac{1}{2}}}{\sqrt{\pi}\left(2^ii!\right)^{\nicefrac{p}{2}}} \int_{-\infty}^\infty \abs{H_i(x)}^p \exp\left(-\left(\frac{p}{4}+\frac{1}{2}\right)x^2\right) dx \\
\geq & \left(\frac{2}{\pi}\right)^{\frac{p}{4} - \frac{1}{2}}\frac{2^{\frac{p}{4}}}{4^pi^{\nicefrac{p}{4}}} \int_{a_i}^{b_i}\exp\left(\left(\frac{p}{4}-\frac{1}{2}\right)x^2\right)dx \\
\geq & \Omega_i\left(i^{-\frac{3}{2}}\int_{a_i}^{b_i}\exp\left(\frac{p-2}{4}x^2\right)dx\right) \\
\geq & \Omega_i\left(i^{-\frac{3}{2}}(b_i-a_i)\exp\left(\frac{p-2}{4}a_i^2)\right)\right) 
\end{align*}

By continuity in $\delta$ we can take $b_i=\sqrt{2i+1}\cos(0)=\sqrt{2i+1}$ and by using the inequality (via the $\cos$ Maclaurin  expansion) $\cos(t)\leq 1 - \frac{t^2}{2} + o(t^2)$ we get
\begin{align*}
b_i-a_i= & \sqrt{2i+1}(1-\cos\left(\sqrt[3]{\frac{7}{8(2i+1)}\pi}\right) \\ 
\geq & \sqrt{2i+1}\left(\frac{1}{2}\sqrt[3]{\frac{7}{8(2i+1)}\pi}^2 - o\left(\sqrt[3]{\frac{7}{8(2i+1)}\pi}^2\right)\right) \\ 
= & \Omega_i(\sqrt{i} i^{-\frac{2}{3}}) = \Omega_i(i^{-\frac{1}{6}}).
\end{align*}

Finally, since for sufficiently large $i$, $a_i^2 > \frac{3}{2}i$ (since the $\cos$ part of $a_i$ tends to 1), for any $p\geq 3$ we obtain
\begin{align*}
    \E\left[\abs{\langle u, Y\rangle}^p\right] = \Omega_i\left(i^{-\frac{3}{2}}i^{-\frac{1}{6}}\exp\left(\frac{p-2}{4} \cdot \frac{3}{2}i\right)\right) = \Omega_i\left(\exp\left(\frac{p-2}{4} \cdot i\right)\right) \underset{i\to\infty}{\longrightarrow} \infty.
\end{align*}
This implies that $\psi$ is not sub-Gaussian.

\section{Background on Dot-Product/Zonal Kernels} \label{appendix:dot-product}

A Kernel $K$ is called a dot product kernel if $K(\x,\x')=h(\x^\top \x')$ for some $h:\R\to\R$ which has a Taylor expansion of the form $h(t)=\sum_{i=0}^\infty a_i t^i$ with $a_i\geq 0$. Importantly, $K$ depends only on $\x^\top \x'$. With inputs uniformly distributed on $\Sphere^{d-1}$, this family of kernels includes the NTK, Laplace kernel, Gaussian (RBF) kernel, and polynomial kernel \citep{minh2006mercer, bietti2020deep, chen2020deep}. We emphasize that for an $L$ layer fully connected network $f(\x;\theta)$, KRR with respect to the corresponding GPK $\mathcal{K}(\x,\x')=\mathbb{E}_{\theta}[f(\x;\theta) \cdot f(\x';\theta)]$ (also called Conjugate Kernel or NNGP Kernel) is equivalent to training the final layer while keeping the weights of the other layers at their initial values \citep{lee2017deep}. Furthermore, KRR with respect to the NTK $\Theta\left(\x,\x'\right)=\mathbb{E}_\theta\left[\left\langle \frac{\partial f\left(\x;\theta\right)}{\partial\theta},\frac{\partial f\left(\x';\theta\right)}{\partial\theta}\right\rangle \right]$ is equivalent to training the entire network \citep{jacot2018neural}. 

Under a uniform distribution on $\Sphere^{d-1}$, the domain of $h$ is $[-1, 1]$, and for any $d\geq 3$ dot-product kernels exhibit the Mercer decomposition 
\begin{align}\label{eq:mercer-dot-product}
    K(\x,\x') = \sum_{\ell=0}^\infty \frac{\hat{\sigma}_\ell}{N(d,\ell)} \sum_{m=1}^{N(d,\ell)}Y_{\ell, m}(\x)Y_{\ell, m}(\x'),
\end{align}
where the eigenfunctions $Y_{\ell, m}$ are the $m$'th spherical harmonic of degree (or frequency) $\ell$, $N(d,\ell)=\frac{2\ell+d-2}{\ell} \binom{\ell+d-3}{d-2}$ is the number of harmonics of each degree, and $\sigma_{\ell}:=\frac{\hat{\sigma}_\ell}{N(d,\ell)}$ are the eigenvalues \citep{smola2000regularization}. Each spherical harmonic can be defined via restrictions of homogeneous polynomials to the unit sphere, with the degree (or frequency) of the spherical harmonic corresponding to the degree of said polynomials.
When $d\gg\ell$, $N(d,\ell)=\Theta_d(d^{\ell})$ and when $\ell \gg d, N(d,\ell)=\Theta_\ell(\ell^{d-2})$. Importantly, all spherical harmonics $Y_{\ell, m}$ of the same degree $\ell$ share the same eigenvalue $\sigma_\ell$, and as a result, there are many repeated eigenvalues. For background on spherical harmonics, see \citet{dai2013approximation, atkinson2012spherical, smola2000regularization}. In order to write the kernel as \eqref{eq:mercer}, we can order $\phi$ in the natural way, by first taking $\tilde{\phi}(\x)=(\sqrt{\sigma_0}Y_{0,1}, \sqrt{\sigma_1}Y_{1,1},\ldots,\sqrt{\sigma_1}Y_{1,N(d,1)}, \sqrt{\sigma_2}Y_{2,1}, \ldots )$, and letting $\phi$ be the same as $\tilde{\phi}$ with zero-valued indices removed (where $\sigma_{\ell}=0$). We let $\psi$ be given accordingly. We note that $\psi_1=Y_{0,1}$ is a constant function.

The famous addition theorem \citep{dai2013approximation}[1.2.8 and 1.2.9] implies that for any $d\geq 3$, $\x,\x'\in\Sphere^{d-1}$ and $\ell\geq 0$,
\begin{align}\label{eq:addition_1}
    \sum_{m=1}^{N(d,\ell)}Y_{\ell, m}(\x)Y_{\ell, m}(\x) = N(d,\ell).
\end{align}
For any $\ell\in\N$, let $N(d,\leq \ell) = \sum_{j=1}^\ell N(d,\ell)$. The addition theorem \eqref{eq:addition_1} in particular implies that the eigenfunctions $\psi_i$ are highly correlated, and definitely not i.i.d. Importantly, \eqref{eq:addition_1} implies that 
\begin{align}\label{eq:beta_alpha_1}
    \text{For any } \ell\in\N, k:=N(d,\leq \ell), \text{ it holds that } \beta_k=\alpha_k=1. 
\end{align}

Furthermore, for any $k\in\N$, let $\ell_k= \max\{\ell \in \N\cup\{0\} \st N(d,\leq\ell) \leq k\}$, so that $N(d,\leq \ell_k) \leq k \leq N(d,\leq \ell_k+1)$. If momentarily we consider the case when $\hat{\sigma}_{\ell}\neq 0$ for all $\ell$, then from \eqref{eq:addition_1}, it holds that for any $\x\in\Sphere^{d-1}$,
\begin{align*}
    \Theta_k(1) = \frac{N(d, \leq \ell_k)}{N(d, \leq \ell_{k+1})} \leq \frac{\norm{\psi_{\leqk}(\x)}^2}{k} \leq \frac{N(d, \leq \ell_{k+1})}{N(d, \leq \ell_{k})} = \Theta_k(1).
\end{align*}
Implying that $\frac{\norm{\psi_{\leqk}(\x)}^2}{k} = \Theta_k(1)$. A similar argument yields
\begin{align*}
    1 - \frac{\hat{\sigma}_{\ell_k}}{\sum_{\ell=\ell_k}^\infty \hat{\sigma}_{\ell}} = \frac{\sum_{\ell=\ell_k+1}^\infty \hat{\sigma}_{\ell}}{\sum_{\ell=\ell_k}^\infty \hat{\sigma}_{\ell}} \leq \frac{\norm{\phi_{\geqk}(\x)}^2}{\tr\left(\Sigma_{\geqk}\right)} \leq \frac{\sum_{\ell=\ell_k}^\infty \hat{\sigma}_{\ell}}{\sum_{\ell=\ell_k+1}^\infty \hat{\sigma}_{\ell}} \leq 1 + \frac{\hat{\sigma}_{\ell_k}}{\sum_{\ell=\ell_k+1}^\infty \hat{\sigma}_{\ell}},
\end{align*}
which analogously to \lemref{lem:int_dim} will typically be $\Theta_k(1)$ if the decay of $\hat{\sigma}$ is at most exponential (but may be slower). This is the case for common kernels such as NTK, Laplace and RBF, and for such kernels we obtain:
\begin{align}
    \alpha_k, \beta_k = \Theta_k(1).
\end{align}

\section{Examples of Kernels That Fit Our Framework}\label{appendix:example_kernels}
Here, we provide some simple examples of kernels that fit our framework. Namely, that $\beta_k$ and possibly $\alpha_k$ (as defined in \defref{def:eigen_combined}) can be bounded. First, note that for each of the terms in \defref{def:eigen_combined}, the denominator is the expected value of the numerator, so $\alpha_k$ and $\beta_k$ quantify how much the features behave as they are "supposed to". Since $\inf\leq \E \leq \sup$, one always has 
\begin{align}\label{eq:alpha_leq_beta}
    0\leq \alpha_k \leq 1 \leq \beta_k.
\end{align}
A control on $\beta_k$ is usually easier than one on $\alpha_k$. Nevertheless, bounding $\alpha_k$ may be made easier by Remark \ref{remark:high_prob}. We also mention that bounds on $\alpha_k, \beta_k$ in one domain can often be extended to others. See \secref{section:high-dim} for details.

\begin{itemize}
    \item Dot-Product Kernels on $\Sphere^{d-1}$: A complete treatment of such kernels is given in \appref{appendix:dot-product}. 

    \item Kernels With Bounded Eigenfunctions: If $\psi_i^2(\x)<M$ for any $i,\x$ the it trivially holds that $\beta_k\leq M$ for any $k\in\N$. Analogously, if $\psi_i^2\geq M'$ then $\alpha_k\geq M'$. This may be weakened to a high probability lower bound (see Remark \ref{remark:high_prob}).

    \item RBF and shift-invariant kernels in $X\subseteq \R^d$: The features $\phi_i$ for an RBF kernel on $X\subseteq \R^d$ with nonempty interior (i.e $X^{\circ}\neq \emptyset$) are given by \citep{steinwart2006explicit}[Theorem 3.7]. If for simplicity $X\subseteq [-1, 1]$,  then $\phi_i$ are bounded, implying that $\psi_i$ are also bounded. Hence, by the previous item, $\beta_k=\bigo_{k,n}(1)$. A simple and easy-to-understand construction of the Mercer Decomposition for general shift-invariant kernels on $[0,1]$ is provided in \citet{mairal2018machine}.

    \item Kernels on the Hypercube $\{-1, 1\}^d$: With a uniform distribution, the hypercube has a Fourier decomposition given by monomials \citep{o2014analysis}. As a result, for kernels of the form $K(\x,\x')=h\left(\frac{\langle \x, \x' \rangle}{\norm{\x}\norm{\x'}}, \frac{\norm{\x}^2}{d}, \frac{\norm{\x'}^2}{d}\right)$ for some $h:\R^3\to\R$, the eigenfunctions $\psi_i$ are given by monomials \citep{yang2019fine}. In particular, for any $i$, $\psi_i^2\equiv 1$ and thus $\alpha_k = \beta_k = 1$ for any $k$.
    
\end{itemize}

\section{Computation of Kernels in Experiments} \label{appendix:experiments}
We plot the variance for a $3$-layer fully connected NTK and polynomial kernel in \figref{fig:multiple_descent} and $3$-layer fully connected GPK in \figref{fig:low_dimensional}. Background on the NTK and GPK is given in \appref{appendix:dot-product}; however, we note here that there is a closed form for the expectations \citep{jacot2018neural, lee2019wide, bietti2020deep}, which we used when computing the figures. First, let
\begin{align*}
    \kappa_0(u):=\frac{1}{\pi}(\pi - \arccos(u)), \qquad \kappa_1(u) := \frac{1}{\pi} \left(u\left(\pi - \arccos(u)\right) + \sqrt{1 - u^2}\right).
\end{align*}
The $L$ layer GPK on $\Sphere^{d-1}$ is equal to 
\begin{align*}
    K_{\text{GPK}}^{(L)}(\x, \x') := \kappa_1\left(K_{\text{GPK}}^{(L-1)}(\x, \x')\right), \qquad K^{(0)}(\x, \x'):= \x^\top \x',
\end{align*}
and the $L$ layer NTK on $\Sphere^{d-1}$ is
\begin{align*}
    \Theta^{(L)}(\x, \x') := \Theta^{(L-1)}(\x, \x')\kappa_0\left(K_{\text{GPK}}^{(L-1)}(\x, \x')\right) + K_{\text{GPK}}^{(L)}(\x, \x').
\end{align*}

\end{document}